\documentclass{article} 
\usepackage{nips15submit_e,times}
\usepackage{comment}
\usepackage{xcolor}
\usepackage{amsfonts}
\usepackage{amsmath,amsthm}
\usepackage{amssymb}
\usepackage{stmaryrd}
\usepackage{accents}
\usepackage{graphicx}
\usepackage{caption}
\usepackage{float}
\usepackage{multirow}
\usepackage{array}
\usepackage{algorithm}
\usepackage[noend]{algpseudocode}
\usepackage{xspace}
\usepackage{wasysym}
\usepackage[caption=false]{subfig}
\usepackage{booktabs}

\newtheorem{theorem}{Theorem}[section]

\theoremstyle{definition}
\theoremstyle{definition}
\theoremstyle{observation}

\renewcommand{\vec}[1]{\mathrm{vec}\left( #1 \right)}

\renewcommand{\algorithmicrequire}{\textbf{Input:}}

\algrenewcommand{\algorithmicforall}{\textbf{for each}}

\graphicspath{{./figures/}}

\title{Cholesky Factor Interpolation for Efficient Approximate Cross-Validation}

\author{
  Da Kuang\\
  University of California, Los Angeles \\
  Los Angeles, CA 90095 \\
  \texttt{da.kuang@cc.gatech.edu}
  \And
  Alex Gittens \\
  International Computer Science Institute \\
  Berkeley, CA 94704 \\
  \texttt{gittens@icsi.berkeley.edu}
  \And
  Raffay Hamid\\
  DigitalGlobe Inc.\\
  Seattle, WA 98103\\
  \texttt{raffay@cc.gatech.edu}
}

\nipsfinalcopy 

\begin{document}

\maketitle

\begin{abstract}

The dominant cost in solving least-square problems using Newton's method is
often that of factorizing the Hessian matrix over multiple values of the
regularization parameter ($\lambda$). We propose an efficient way to 
interpolate the Cholesky factors of the Hessian matrix computed over a small
set of $\lambda$ values. This approximation enables us to optimally minimize the
hold-out error while incurring only a fraction of the cost compared to exact
cross-validation. We provide a formal error bound for our approximation scheme
and present solutions to a set of key implementation challenges that allow our
approach to maximally exploit the compute power of modern architectures.
We present a thorough empirical analysis over multiple datasets to show the
effectiveness of our approach.

\end{abstract}

\section{Introduction}
\label{sec:introduction}

\noindent Least-squares regression has continued to maintain its significance as a worthy opponent to more advanced learning algorithms. This is mainly because:

\vspace{0.025cm}\noindent ${\textbf{a}-}$ Its closed form solution can be found
efficiently by maximally exploiting modern hardware using high performance
BLAS-$3$ software~\cite{golub2012matrix}.

\vspace{0.025cm}\noindent ${\textbf{b}-}$ Advances in kernel
methods~\cite{scholkopf1999advances}~\cite{le2013fastfood}~\cite{kar2012random}
can efficiently construct non-linear spaces, where the closed form solution of
linear regression can be readily used.

\vspace{0.025cm}\noindent ${\textbf{c}-}$ Availability of error correcting
codes~\cite{dietterich1995solving} allow robust simultaneous learning of
multiple classifiers in a single pass over the data.

\vspace{0.125cm}\noindent However, an important bottleneck in solving large
least-squares problems is the cost of $k$-fold cross
validation~\cite{alpaydin2004introduction}. To put this cost in perspective, we
show in Figure~\ref{fig:sse_pipeline} the costs of performing the main steps in
solving least-squares using Newton's method.

Specifically, for $d$-dimensional data each fold requires finding optimal value
of regularization parameter $\lambda$ searched over $q$ values. This requires
solving a linear system with $d$ variables, represented by the $d \times d$
Hessian matrix, $q$ times for each of the $k$ folds. As solving this system
using the Cholesky factorization of the Hessian costs $\mathcal{O}(d^{3})$
operations, the total cost of $k$-folds adds up to $\mathcal{O}(kqd^{3})$
operations. Other considerable costs include computing the Hessian requiring
$\mathcal{O}(nd^{2})$ operations. Comparing these costs, we see that when $n <
kqd$, cross validation is the dominant cost. Figure~\ref{fig:timing_info}
presents an empirical sense of cross-validation and Hessian costs as a function
of $n$ and $d$.

\begin{figure}[t]
\centering
\includegraphics[width=0.75\textwidth]{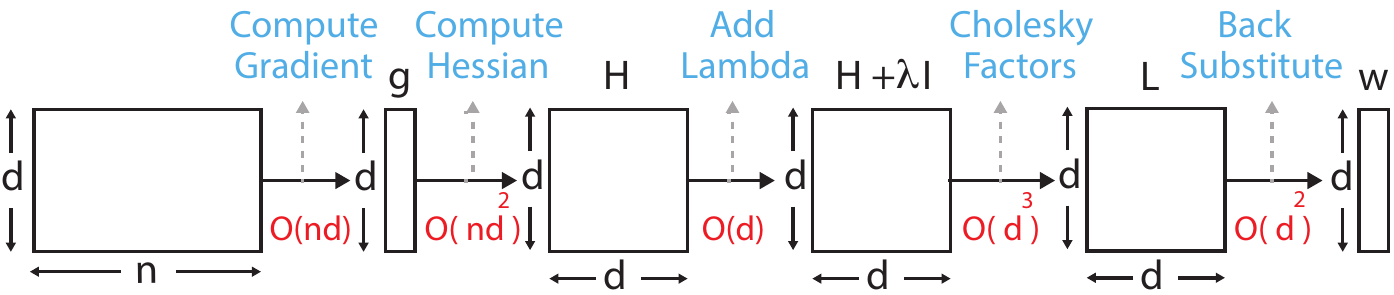}
  \caption{{\small Computational steps for least squares regression.}}
  \label{fig:sse_pipeline}
  \vspace{-0.3cm}
\end{figure}

\begin{figure*}[b]
\centering
\includegraphics[width=1.0\textwidth]{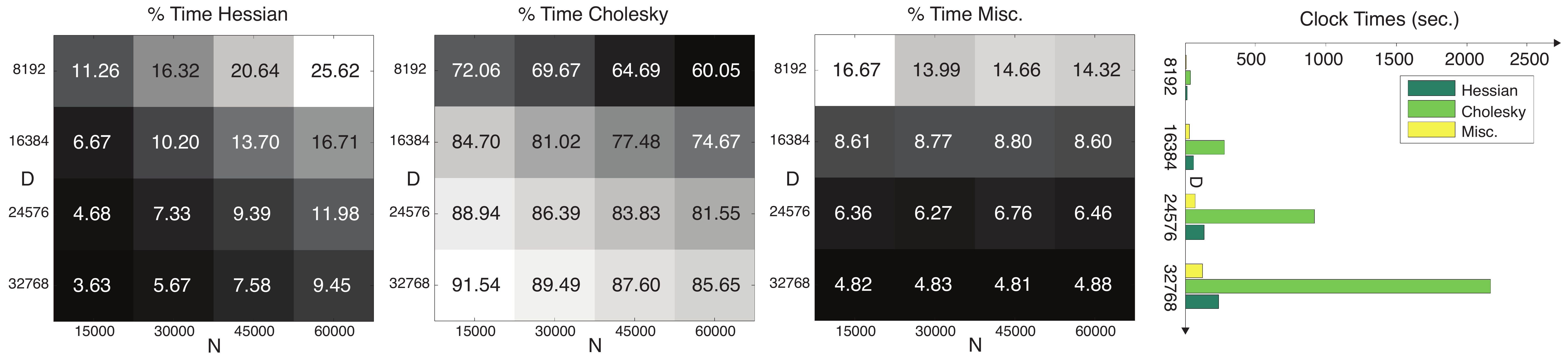}
  \caption{{\small Percent times taken by the three main steps in least-squares
      pipeline for MNIST~\cite{lecunmnist} data projected using polynomial
      kernel~\cite{kar2012random} to different sized feature spaces. In Figures
      a, b, and c, the x-axis represents the number of training points, while
      the y-axis represents the size of the feature space.}}
  \label{fig:timing_info}
  \vspace{-0.3cm}
\end{figure*}

Our goal is to reduce the computational cost of cross-validation without
increasing the hold-out error. To this end, we propose to densely interpolate
Cholesky factors of the Hessian matrix using a sparse set of $\lambda$ values.
Our key insight is that Cholesky factors for different $\lambda$ values lie on
smooth curves, and can hence be approximated using polynomial functions (see
Figure~\ref{fig:pichol_teaser} for illustration). We provide empirical evidence
supporting this intuition and provide an error bound for this approximation.

\begin{figure}
\centering
\includegraphics[width=0.45\textwidth]{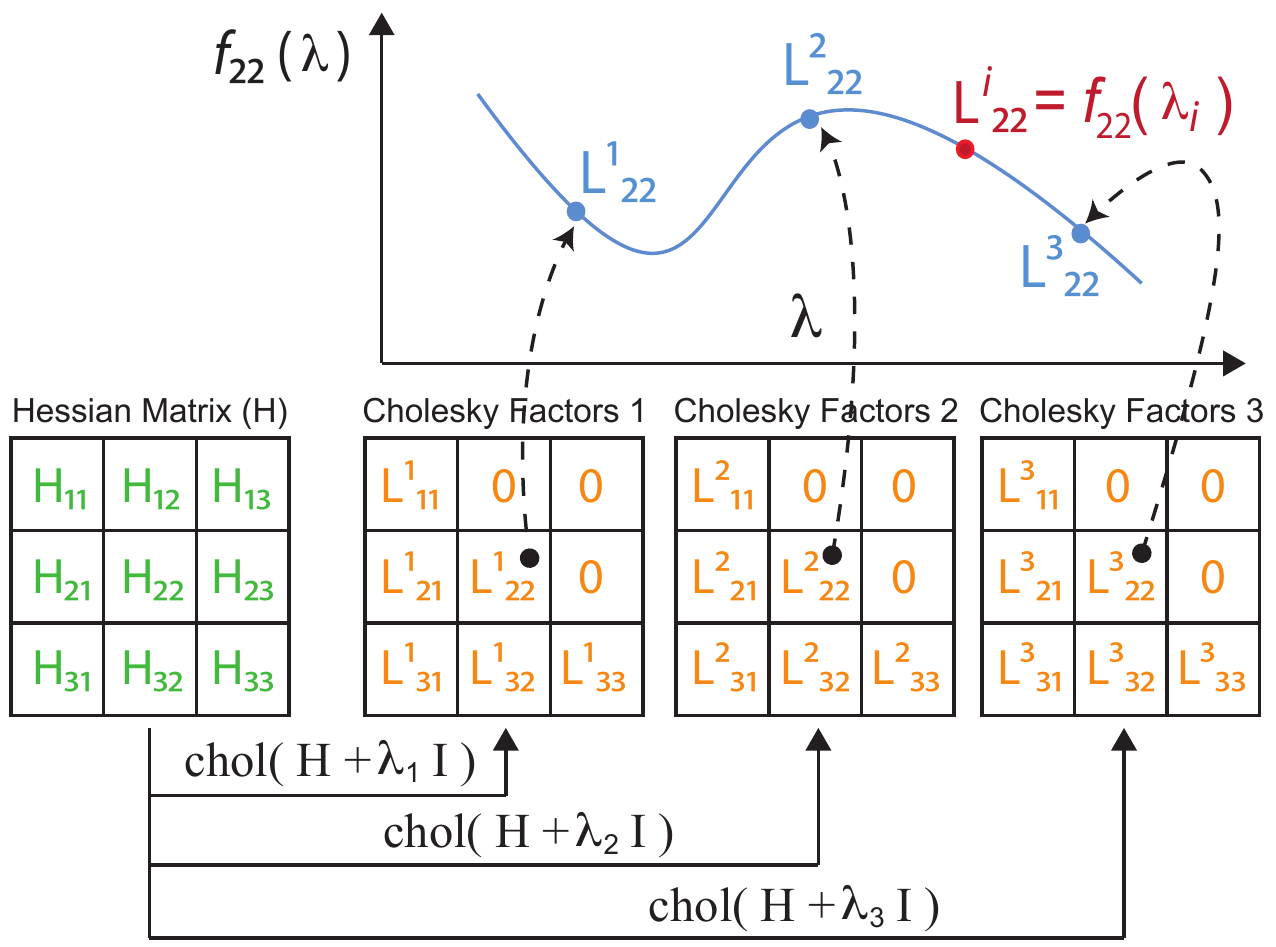}
  \caption{{\small Given a Hessian matrix $\mathbf{H}$, we compute the Cholesky
      factors $\mathbf{L}^s$ of $\mathbf{H} + \lambda_s\mathbf{I}$ for a small
      set $\{\lambda_s\}$. For each $(p,q)$-th entry in these Cholesky factors
      ($1 \leq p,q \leq d$), we fit a polynomial function $f_{p,q}(\lambda)$.
      These functions are used to estimate entries of $\mathbf{L}$ for a larger
      set of $\lambda$ values.}}
  \label{fig:pichol_teaser}
  \vspace{-0.3cm}
\end{figure}

We formalize a least-squares framework to simultaneously learn the multiple
polynomial functions required to densely interpolate the Hessian factors. An
important challenge to solve this problem efficiently is a matrix-vector
conversion strategy with minimal unaligned memory access (see
$\S$~\ref{sec:ichallenges} for details). To address this challenge, we propose
a general-purpose strategy to efficiently convert blocks of Choelsky factors to
their corresponding vectors. This enables our learning framework to use
BLAS-$3$~\cite{golub2012matrix} level computations, therefore maximally
exploiting the compute power of modern hardware. Our results demonstrate that
the proposed approximate regularization approach offers a significant
computational speed-up while providing hold-out error that is comparable to
exact regularization.

\section{Related Work}

\noindent Solving linear systems has been well-explored in terms of their type~\cite{golub2012matrix}~\cite{saad2003iterative} and scale~\cite{kim2007interior}~\cite{van2003iterative}. Our focus is on the least-squares problem as the computational basis of linear regression. Popular methods to solve least-squares include QR~\cite{golub2012matrix}, Cholesky~\cite{horn2012matrix}, and Singular Value Decomposition (SVD)~\cite{golub1965calculating}. For large dense problems, Cholesky factorization has emerged as the method of choice which is used to solve the normal equation, \textit{i.e.}, the linear system represented by the Hessian matrix of least-squares problem. This is because of its storage ($\times 2$) and efficiency ($\times 2$ and $\times 39$) advantages over QR and SVD respectively~\cite{golub2012matrix}).

For real-world problems, it is common for system parameters to undergo change over time. Previous works in this context have mostly focused on low-rank updates in system parameters including direct matrix inverse~\cite{hager1989updating}, LU decomposition~\cite{kaess2011isam2}, Choleksy factorization~\cite{koutis2012fast}, and Singular Value Decomposition~\cite{bunch1978updating}~\cite{gu1995downdating}. Our problem is however different in two important ways: (i) we focus on linear systems undergoing full rank updates, and (ii) our updates are limited to the diagonal of the Hessian. Both of these attributes are applicable to the regularization of least-squares problems, as we shall see in $\S$~\ref{sec:pichol_framework}.

A standard way to solve regularized least-squares is to find the SVD of the input matrix once for each training fold, and then reuse the singular vectors for different $\lambda$ values. For large problems however, finding the SVD of a design matrix even once can be prohibitively expensive. In such situations, truncated~\cite{hansen1987truncatedsvd} or randomized approximate~\cite{halko2011finding} SVD can be used. However their effectiveness for optimizing hold-out error for least squares problems is still unexplored. There has also been work to reduce the number of regularization folds~\cite{cawley2004fast} by minimizing the regularizing path~\cite{efron2004least}. Our work can be used in conjunction with these approaches to further improve the overall performance. 
\section{\textit{pi}Cholesky Framework}
\label{sec:pichol_framework}
\subsection{Preliminaries}
\label{subsec:preliminaries}

\noindent Let $\mathbf{X}$ be the $n \times (d+1)$ design matrix with each row as one of the $n$ training examples in a $d$-dimensional space. Let $\mathbf{y}$ be the $n$ dimensional vector for training labels. Then the Tychonov regularized~\cite{tikhonov1943stability} least-squares cost function $\mathbf{J(\theta)}$ is given as:
\begin{equation}
\begin{small}
\mathbf{J(\theta)} = \frac{1}{2}(\mathbf{y} - \mathbf{X\theta})^{\top}(\mathbf{y} - \mathbf{X\theta}) + \frac{\lambda}{2}\mathbf{\theta}^{\top}\mathbf{\theta}
\end{small}
\end{equation}
\noindent Setting the derivative of $\mathbf{J(\theta)}$ with respect to $\mathbf{\theta}$ equal to zero results in the following solution of $\theta$:
\begin{equation}
\begin{small}
\mathbf{\theta} = (\mathbf{H} + \lambda \mathbf{I})^{-1}\mathbf{g}
\label{eq:reg_normal_equation_Hg}
\end{small}
\end{equation}
\noindent where $\mathbf{I}$ is the $(d+1) \times (d+1)$ identity matrix, the Hessian matrix $\mathbf{H}$ equals $\mathbf{X}^{\top}\mathbf{X}$ and the gradient vector $\mathbf{g}$ is $\mathbf{X}^{\top}\mathbf{y}$. Equation~\ref{eq:reg_normal_equation_Hg} is solved for different values of $\lambda$ in a k-fold cross-validation setting, and the $\lambda$ with minimum hold-out error in expectation is picked.

\subsection{Least Squares using Cholesky Factors}
\label{subsec:exactcholesky}

\vspace{0.1cm}\noindent Rewriting Equation~\ref{eq:reg_normal_equation_Hg} as $\mathbf{A}\theta = \mathbf{g}$, where $\mathbf{A}$ $=$ $\mathbf{H} + \lambda \mathbf{I}$, we can find Cholesky factors of $\mathbf{A}$ as $\mathbf{A} = \mathbf{L}\mathbf{L}^{\textrm{T}}$, where $\mathbf{L}$ is a lower-triangular matrix. Equation~\ref{eq:reg_normal_equation_Hg} can now be written as $\mathbf{L}\mathbf{L}^{\textrm{T}}\theta = \mathbf{g}$. Using $\mathbf{L}^{\textrm{T}}\theta$ as $\mathbf{w}$, this can be solved by a forward-substitution to solve the triangular system $\mathbf{L}\mathbf{w} = \mathbf{g}$ for $\mathbf{w}$ followed by a back-substitution to solve the triangular system $\mathbf{L}^{\textrm{T}}\theta = \mathbf{w}$ for $\theta$.

\begin{figure}
\centering
\includegraphics[width=0.65\textwidth]{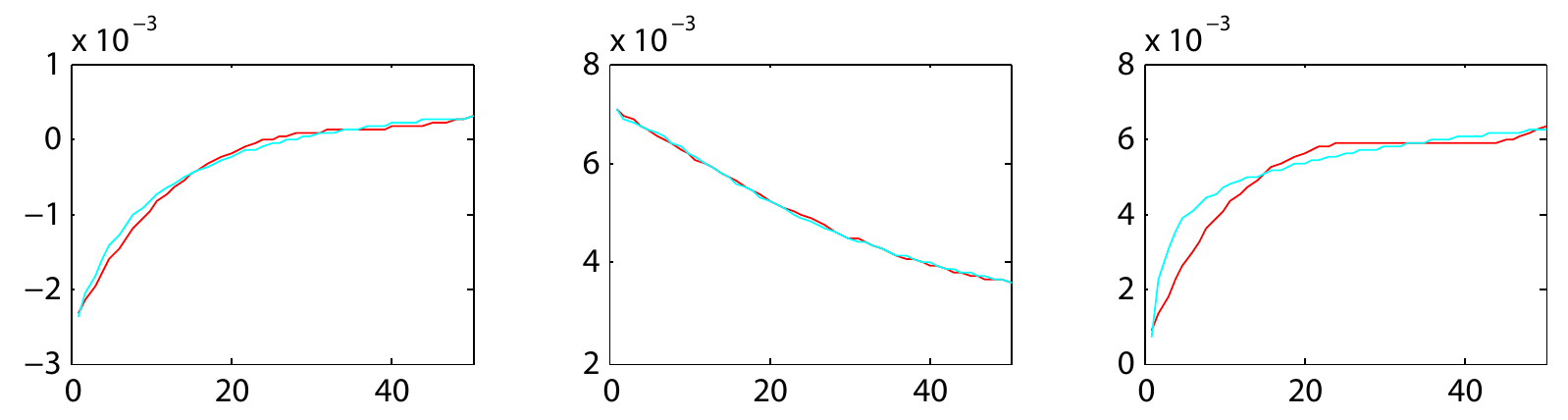}
  \caption{{\small Plots for subset of entries in $\mathbf{L}$ (y-axes) for different $\lambda$ values (x-axes) on MNIST data. The red curves use exact Cholesky over $50$ $\lambda$ values, while the blue curves show the interpolated values of $\mathbf{L}$ evaluated from $2^{\textrm{nd}}$ order polynomials learned using $6$ values of $\lambda$.}}
  \label{fig:exact_inter_factors}
  \vspace{-0.3cm}
\end{figure}

\subsection{Cholesky Factors Interpolation}

\vspace{0.1cm}\noindent To minimize the computational cost for cross-validation, we propose to compute Cholesky factors $\mathbf{L}\mathbf{L}^{\textrm{T}}$ of $\mathbf{A}$ over a small set of $\lambda$ values, followed by interpolating the corresponding entries in $\mathbf{L}$ for a more densely sampled set of $\lambda$ values. Note that corresponding entries of $\mathbf{L}$ for different values of $\lambda$ lie on smooth curves that can be accurately approximated using multiple polynomial functions. Figure~\ref{fig:exact_inter_factors} illustrates this point empirically, where the red curves show multiple corresponding entries in $\mathbf{L}$ computed over different $\lambda$ values using exact Cholesky for MNIST data.

Recall that for $d$-dimensional data, there are $\textrm{D} = (d+1)(d+2)/2$ number of entries in the lower-triangular part of $\mathbf{L}$. Therefore, to interpolate a Cholesky factor based on those computed for different values of $\lambda$, we need to learn $\textrm{D}$ polynomial functions, each for an entry in the lower-triangular part of $\mathbf{L}$.

This challenge can be posed as a least-squares problem. Recall that each polynomial function to be learned is of order $r$. We therefore need $g > r$ exact Cholesky factors, each of which is computed using one of the $g$ values of $\lambda$. We evaluate a polynomial basis of the space of $r$-th order polynomials at the sparsely sampled $g$ values of $\lambda$ (\textit{e.g.}, $1,\lambda$, and $\lambda^2$ for second-order polynomials); this way, we form our $g \times (r+1)$ observation matrix $\mathbf{V}$. Our targets are the $g$ rows of $\textrm{D}$ values, where each row corresponds to the exact Cholesky matrix computed for one of the $g$ values of $\lambda$. This forms our $g \times \textrm{D}$ target matrix $\mathbf{T}$.

In Algorithm~\ref{alg:chol_interp}, we use monomials as the polynomial basis to interpolate the $r$th-order polynomials. While we can employ other polynomial bases that are numerically more stable (such as Chebyshev polynomials), we found in our experiments that the observation matrix $\mathbf{V}$ is well-conditioned and therefore using monomials does not harm our numerical stability.

\begin{algorithm}[t]
\caption{{\sc -- \textit{pi}Cholesky}} \label{alg:chol_interp}
{\bf Input:} Degree $r$, $\{\lambda_{s}\}$ for $s =\{1, 2, \cdot \cdot g\}$, $g > r$\\
{\bf Output:} A $(r+1) \times \textrm{D}$ coefficient matrix $\Theta$, where $\textrm{D} = (d+1)(d+2)/2$ and $d$ is the data dimensionality\\\vspace{-0.25cm}
\begin{algorithmic}[1]
    \State Find $\mathbf{L}^{s} = \textrm{chol}(\mathbf{H} + \lambda_{s}\mathbf{I})$ for $s =\{1, 2, \cdot \cdot \cdot g\}$
    \State Convert each $\mathbf{L}^{s}$ into a row vector to construct the $g \times \textrm{D}$ target matrix $\mathbf{T}$
        \State Find the $g \times g$ Vandermonde matrix $\mathbf{W}$
        \State Extract the leftmost $(r+1)$ colums of $\mathbf{W}$ to form the $g \times (r+1)$ observation matrix $\mathbf{V}$
        \State Find $\mathbf{G}_{\lambda} = \mathbf{V}^{\textrm{T}}\mathbf{T}$ and $\mathbf{H}_{\lambda} = \mathbf{V}^{\textrm{T}}\mathbf{V}$
        \State Find $\Theta = \mathbf{H}_{\lambda}^{-1}\mathbf{G}_{\lambda}$
\end{algorithmic}
\end{algorithm}

\vspace{0.1cm} \noindent We can now define our cost function $\mathbf{J}_{\lambda}(\Theta)$ as:
\begin{equation}
\begin{small}
\mathbf{J}_{\lambda}(\Theta) = \frac{1}{2}(\mathbf{T} - \mathbf{V}\Theta)^{\top}(\mathbf{T} - \mathbf{V}\Theta)
\end{small}
\end{equation}
\noindent where $\Theta$ is the $(r+1) \times \textrm{D}$ polynomial coefficient matrix. Each column of $\Theta$ represents the $(r+1)$ coefficients of the $\textrm{D}$ polynomial functions. Following procedure similar to $\S$~\ref{subsec:preliminaries}, the expression of $\Theta$ can be written as
\begin{equation}
\begin{small}
\Theta = \mathbf{H}_{\lambda}^{-1}\mathbf{G}_{\lambda}
\end{small}
\end{equation}
\noindent Here $\mathbf{H}_{\lambda} = \mathbf{V}^{\textrm{T}}\mathbf{V}$, and $\mathbf{G}_{\lambda} = \mathbf{V}^{\textrm{T}}\mathbf{T}$.

Given a new regularization parameter value $\lambda_{t}$, the value for $\mathbf{L}^{t}$ can be computed by evaluating the $\textrm{D}$ polynomial functions at $\lambda_{t}$. This procedure is listed in Algorithm~\ref{alg:chol_interp}. The interpolation results using Algorithm~\ref{alg:chol_interp} for the Cholesky factors computed on the MNIST data~\cite{lecunmnist} are shown with blue curves in Figure~\ref{fig:exact_inter_factors}. Here we set $g=5$ and $r=2$. As can be seen from the figure, the blue plots (interpolated) trace the red plots (exact) closely.

\vspace{0.1cm}\noindent \textbf{Computational Complexity}: The dominant step of Algorithm~\ref{alg:chol_interp} is evaluating $\mathbf{L}^{s}$ for $s = 1,2, \cdots, g$, which requires $\mathcal{O}(gd^{3})$ operations. The only other noteworthy steps of Algorithm~\ref{alg:chol_interp} are finding $\mathbf{G}_{\lambda}$ and $\Theta$ each of which takes $\mathcal{O}(grd^{2})$ operations. Since $d >> r$, the overall asymptotic cost of  Algorithm~\ref{alg:chol_interp} is $\mathcal{O}(gd^{3})$. Furthermore, it only takes $\mathcal{O}(rd^{2})$ operations to evaluate the interpolated Cholesky factor $\mathbf{L}^{t}$ for each $\lambda_{t}$ value.

\section{Theoretical Analysis}
\label{sec:theory}

\newcommand{\unvec}[1]{\ensuremath{\textrm{vec}^{-1}(#1)}}

There are two challenges in developing a reasonable bound on the error of the
\textit{pi}Cholesky algorithm. The first is determining the extent to which the
Cholesky factorization can be approximated entrywise as a polynomial: if one
explicitly forms the symbolic Cholesky factorization of even a $3 \times 3$
matrix, it is not clear at all that the entries of the resulting Cholesky factor
can be approximated well by any polynomial. The second is determining the extent
to which the \emph{particular} polynomial recovered by the \textit{pi}Cholesky
procedure of solving a least squares system (Algorithm~\ref{alg:chol_interp})
is a good polynomial approximation to the Cholesky factorization.

The classical tool for addressing the first challenge is the Bramble-Hilbert
lemma~\cite{BrennerScott:FEM}[Lemma 4.3.8], which guarantees the existence of a
polynomial approximation to a smooth function on a compact domain, with Sobolev
norm approximation error bounded by the Sobolev norm of the derivatives of the
function itself. Our proof of the existence of a polynomial approximation to the
Cholesky factorization is very much in the spirit of the Bramble-Hilbert lemma,
but provides sharper results than the direct application of the lemma itself
(essentially because we do not care about the error in approximating the
derivatives of the Cholesky factorization). We surmount the second obstacle by
noting that, if instead of sampling the Cholesky factorization itself when
following the \textit{pi}Cholesky procedure, we sample from a polynomial
approximation to the Cholesky factorization, then the error in the resulting
polynomials can be bounded using results on the stability of the solution to
perturbed linear systems.

In our arguments, we find it convenient to use the Fr\'echet derivative: given
a mapping $f : X \rightarrow Y$ between two normed linear linear spaces we
define $Df$, the derivative of $f$, to be the function that maps $\mathbf{u}
\in X$ to $D_{\mathbf{u}}f: X \rightarrow Y$, the unique linear map (assuming
it exists) that is tangent to $f$ at $\mathbf{u}$ in the sense that
\[
    \lim_{\|\pmb{\delta}\|_X \rightarrow 0} \frac{\|f(\mathbf{u} + \pmb{\delta}) - f(\mathbf{u}) - D_{\mathbf{u}}f(\pmb{\delta})\|_Y}{\|\pmb{\delta}\|_X} = 0,
\]
The Fr\'echet derivative generalizes the conventional derivative, so it follows
a chain rule, can be used to form Taylor Series expansions, and shares the
other properties of the conventional derivative~\cite{Marsden}.

We inductively define the $r$th derivative of $f,$ $D^rf$, as the function that
maps $\mathbf{u} \in X$ to the unique linear map tangent to
$D^{r-1}_{\mathbf{u}}f$ at $\mathbf{u}$ in the sense that
\[
    \lim_{\|\pmb{\delta}_1\|_X \rightarrow 0} \frac{\|D^{r-1}_{\mathbf{u} + \pmb{\delta}_1}f(\pmb{\delta}_2) - D^{r-1}_{\mathbf{u}}f(\pmb{\delta}_2) - D^r_{\mathbf{u}}f(\pmb{\delta}_1, \pmb{\delta}_2)\|_Y}{\|\pmb{\delta}_1\|_X} = 0.
\]

For a comprehensive introduction to Fr\'echet derivatives and their properties, we refer the reader to~\cite{Marsden}.

\subsection{Performance Guarantee for \textit{pi}Cholesky}

Let $\mathbf{p}_{\mathrm{TS}}(\lambda; \lambda_{\textrm{c}})$ denote the
second-order polynomial obtained from the Taylor Series expansion of
$\mathcal{C}(\mathbf{A} + \lambda \mathbf{I})$ around $\lambda =
\lambda_{\textrm{c}}$ and let $\mathbf{p}_\pi(\lambda)$ denote the
approximation to $\mathcal{C}(\mathbf{A} + \lambda \mathbf{I})$ obtained using
the \textit{pi}Cholesky procedure. Our argument consists in bounding the errors
in approximating $\mathcal{C}$ with $\mathbf{p}_{\mathrm{TS}}$
(Theorem~\ref{thm:tsapprox}) and in approximating $\mathbf{p}_{\mathrm{TS}}$
with $\mathbf{p}_\pi$ (Theorem~\ref{thm:interpstability}). The root
mean-squared error in approximating $\mathcal{C}(\mathbf{A} + \lambda
\mathbf{I})$ with the \textit{pi}Cholesky procedure is then controlled using
the triangle inequality (Theorem~\ref{thm:mainresult}): if $\gamma = |\lambda -
\lambda_{\textrm{c}}|$ and $w = \max_i |\lambda_i - \lambda_{\textrm{c}}|$ is
the maximum distance of any of the sample points used in
Algorithm~\ref{alg:chol_interp} from $\lambda_{\textrm{c}}$, then

\begin{equation}
  \label{eqn:mainresult}
   \frac{1}{\sqrt{D}} \| \mathcal{C}(\mathbf{A} + \lambda \mathbf{I}) - 
   \mathbf{p}_{\pi}(\lambda) \|_F \leq \big[ \gamma^3 + 
   \sqrt{g} w^3 (1 + \gamma^2) (\lambda_{\textrm{c}} + 1) \|\mathbf{V}^\dagger\|_2 \big]
   \frac{\mathrm{R}_{[\lambda_{\textrm{c}} - \gamma, \lambda_{\textrm{c}} + \gamma]}}{\sqrt{D}}
\end{equation}

Here, $g$ is the number of sample points (values of $\lambda$) used in the
\textit{pi}Cholesky procedure, $D = (d+1)(d+2)/2$ is the number of elements in
$\mathcal{C}(\mathbf{A} + \lambda \mathbf{I})$, and $\mathrm{R}_{[a,b]}$ is defined in
Theorem~\ref{thm:tsapprox}. The quantity $\mathrm{R}_{[a,b]}$ measures the
magnitude of the third-derivative of $\mathcal{C}(\mathbf{A} +
\lambda\mathbf{I})$ over the interval $\lambda \in [a,b]$; when it is small,
the implicit assumption made by the \textit{pi}Cholesky procedure that
$\mathcal{C}$ is well-approximated by some quadratic polynomial is reasonable. Unfortunately
the task of relating $\mathrm{R}_{[a,b]}$ to more standard quantities such as the 
eigenvalues of $\mathbf{A}$ is beyond the reach of our analysis.

The matrix $\mathbf{V}$ (defined in Algorithm~\ref{alg:chol_interp}) is a
submatrix of the Vandermonde matrix formed by the sample points, so the quantity
$\|\mathbf{V}^\dagger\|_2$ measures the conditioning of the sample points used
to fit the \textit{pi}Cholesky polynomial approximant: it is small when the $g$
rows of $\mathbf{V}$ are linearly independent. This is exactly the setting in which we
expect the least-squares fit in Algorithm~\ref{alg:chol_interp} to be most stable.

The cubic dependence on $\gamma = |\lambda - \lambda_{\textrm{c}}|$ in our bound reflects our 
intuition that since  the Cholesky factorization is nonlinear, we do not expect the quadratic
approximation formed using the \textit{pi}Cholesky procedure to perform well
far away from the interpolation points used to form the approximation. The cubic dependence
on $w = \max_{i} | \lambda_i - \lambda_{\textrm{c}}|$ also captures our intuition that
we can only expect \textit{pi}Cholesky to give a good approximation when the 
interpolation points used to fit the interpolant cover a small interval
containing the optimal regularization parameter.

To algorithmically address the fact that this optimal
$\lambda_{\textrm{c}}$ is unknown, we introduce a Multi-level Cholesky
procedure in Section~\ref{exp:algs} that applies a binary-search-like procedure
to narrow the search range before applying \text{pi}Cholesky.

\subsection{Proof of the performance guarantee for \textit{pi}Cholesky}

Our first step in developing the Taylor series of $\mathcal{C}(\mathbf{A})$ is
establishing that $\mathcal{C}$ is indeed Fr\'echet differentiable, and finding
an expression for the derivative.
\begin{theorem}
\label{thm:derivofcholesky}
 Assume $\mathbf{A}$ is a positive-definite matrix and $\mathcal{C}(\mathbf{A}) = \mathbf{L}.$ Then $D_{\mathbf{L}} \mathcal{S}(\pmb{\Gamma}) = \pmb{\Gamma} \mathbf{L}^T + \mathbf{L} \pmb{\Gamma}^T$ for any lower-triangular matrix $\pmb{\Gamma}$ and $D_{\mathbf{L}} \mathcal{S}$ is full-rank. Furthermore, if $\mathbf{\Delta}$ is a symmetric matrix, then
 \[
  D_{\mathbf{A}} \mathcal{C}(\pmb{\Delta})= \left(D_{\mathbf{L}} \mathcal{S}\right)^{-1}(\pmb{\Delta}).
 \]
\end{theorem}

To key to establishing this result is the observation that $\mathcal{C}$ is the
inverse of the mapping $\mathcal{S} : \mathbf{L} \mapsto \mathbf{LL}^T$ that
maps the set of lower-triangular matrices into the set of symmetric matrices.
Differentiability of $\mathcal{C}$ is then a consequence of the following
corollary of the Inverse Function Theorem (Theorem~2.5.2 of \cite{Marsden}).
\begin{theorem}
    \label{thm:chainrule}
    Let $f : X \rightarrow Y$ be continuously differentiable. If $Df(\mathbf{u})$ is invertible, then
    \[
        D(f^{-1})_{f(\mathbf{u})} = \left(Df_{\mathbf{u}}\right)^{-1}.
    \]
\end{theorem}

\begin{proof}[Proof of Theorem~\ref{thm:derivofcholesky}]
    By Theorem~\ref{thm:chainrule} and the fact that $\mathcal{C} = \mathcal{S}^{-1},$ it suffices to establish that
    $\mathcal{S}$ is continuously differentiable with $D\mathcal{S}_{\mathbf{L}}(\pmb{\Gamma}) = \pmb{\Gamma} \mathbf{L}^T + \mathbf{L} \pmb{\Gamma}^T$ and
    that $D\mathcal{S}_{\mathbf{L}}$ is invertible.

    Recall that (assuming it exists), $D\mathcal{S}_{\mathbf{L}}$ is the
    unique linear map tangent to $\mathcal{S}$ at $\mathbf{L}$ and observe that
    \begin{align*}
        \mathcal{S}(\mathbf{L} + \pmb{\Gamma}) & = \mathbf{LL}^T + \pmb{\Gamma} \mathbf{L}^T + \mathbf{L} \pmb{\Gamma} + \pmb{\Gamma}\pmb{\Gamma}^T \\
                                               & = \mathcal{S}(\mathbf{L}) + \pmb{\Gamma} \mathbf{L}^T + \mathbf{L} \pmb{\Gamma} + \mathrm{O}(\|\pmb{\Gamma}\|^2).
    \end{align*}
    It follows that $D\mathcal{S}_{\mathbf{L}}$ exists and is as stated. Clearly $D\mathcal{S}_{\mathbf{L}}$ is also continuous as a function of $\mathbf{L},$ so $\mathcal{S}$ is continuously differentiable.

    To show that $D\mathcal{S}_{\mathbf{L}}$ is invertible, assume that $\pmb{\Gamma}$ is such that $D\mathcal{S}_{\mathbf{L}}(\pmb{\Gamma}) = \pmb{\Gamma} \mathbf{L}^T + \mathbf{L} \pmb{\Gamma}^T = \mathbf{0}.$ Because $\mathbf{A}$ is positive-definite, $\mathbf{L}$ is invertible, and we can conclude that $\mathbf{L}^{-1} \pmb{\Gamma} = -\pmb{\Gamma}^T (\mathbf{L}^{-1})^T.$ The left hand side is a lower triangular matrix since $\mathbf{L}^{-1}$ and $\pmb{\Gamma}$ are lower-triangular; for similar reasons, the the right hand side is upper-triangular. It follows that $\mathbf{L}^{-1} \pmb{\Gamma}$ is a diagonal matrix, and hence $\pmb{\Gamma} = \mathbf{L} \mathbf{D}$ for some diagonal matrix $\mathbf{D}.$ Together with the assumption that $\pmb{\Gamma}\mathbf{L}^T + \mathbf{L} \pmb{\Gamma}^T = \mathbf{0},$ this implies that $2 \mathbf{L D L}^T = \pmb{\Gamma} \mathbf{L}^T + \mathbf{L} \pmb{\Gamma}^T = \mathbf{0}$, and consequently $\mathbf{D} = \mathbf{0}.$ Thus $\pmb{\Gamma} = \mathbf{0},$ so we have established that the nullspace of $D\mathcal{S}_{\mathbf{L}}$ is $\mathbf{0}.$ It follows that $D\mathcal{S}_{\mathbf{L}}$ is invertible.

    The claims of Theorem~\ref{thm:derivofcholesky} now follow.
\end{proof}

The higher-order derivatives of $\mathcal{C}$ are cumbersome, so instead of dealing directly with $\mathcal{C}$, which maps matrices to matrices, we compute the higher-order derivatives of the equivalent function $C = \mathrm{vec} \circ \mathcal{C} \circ \mathrm{vec}^{-1}$ that maps vectors to vectors. The following theorem gives the first three derivatives of $C.$

\begin{theorem}
    \label{thm:derivsofC}
    Let $X, Y \subset \mathbb{R}^{d^2}$ be the image under $\vec{\cdot}$ of, respectively, the set of positive-definite matrices of order $d$ and the space of lower-triangular matrices of order $d.$ Define $C: X \rightarrow Y$ by
    \[
        C = \mathrm{vec} \circ \mathcal{C} \circ \mathrm{vec}^{-1}.
    \]

    When $\mathbf{A}$ is positive-definite, the first three derivatives of $C$ at $\mathbf{v_A}$ are given by
    \begin{align*}
        D_{\mathbf{v_A}}C(\pmb{\delta}_1) & = \mathbf{M}^{-1} \pmb{\delta}_1, \\
        D_{\mathbf{v_A}}^2C(\pmb{\delta}_1, \pmb{\delta}_2) &= -\mathbf{M}^{-1} \cdot \llbracket \mathrm{vec}^{-1}(\mathbf{M}^{-1} \pmb{\delta}_1)\rrbracket \cdot \mathbf{M}^{-1} \pmb{\delta}_2, \\
        D_{\mathbf{v_A}}^3C(\pmb{\delta}_1, \pmb{\delta}_2, \pmb{\delta}_3) &= \mathbf{M}^{-1} \cdot \Big( \llbracket \mathrm{vec}^{-1}( \mathbf{M}^{-1} \pmb{\delta}_1) \rrbracket \cdot \mathbf{M}^{-1} \cdot \llbracket \mathrm{vec}^{-1}(\mathbf{M}^{-1} \pmb{\delta}_2) \rrbracket \\
        & \quad\quad\quad\quad+ \llbracket \mathrm{vec}^{-1}( \mathbf{M}^{-1} \cdot \llbracket \mathrm{vec}^{-1}(\mathbf{M}^{-1} \pmb{\delta}_1)\rrbracket \cdot \mathbf{M}^{-1}\pmb{\delta}_2)\rrbracket \\
        & \quad\quad\quad\quad+ \llbracket \mathrm{vec}^{-1}(\mathbf{M}^{-1}\pmb{\delta}_2) \rrbracket \cdot \mathbf{M}^{-1} \cdot \llbracket \mathrm{vec}^{-1}(\mathbf{M}^{-1}\pmb{\delta}_1))\rrbracket \Big) \cdot \mathbf{M}^{-1} \pmb{\delta}_3,
\end{align*}
    where $\mathbf{M} = \llbracket \mathcal{C}(\mathbf{A}) \rrbracket.$
\end{theorem}

\begin{proof}
    First we convert the expression for $D_{\mathbf{A}}\mathcal{C}$ given in Theorem~\ref{thm:derivofcholesky} into an expression for $D_{\mathbf{v_A}}C.$ To do so, we note that $\vec{\cdot}$ and $\vec{\cdot}^{-1}$ are linear functions, so are their own derivatives. It follows from the Chain Rule (Theorem~2.4.3 of \cite{Marsden}) that
    \[
        DC = \mathrm{vec} \circ D\mathcal{C} \circ \mathrm{vec}^{-1},
    \]
    so
    \[
        D_{\mathbf{v_A}}C(\mathbf{v}_{\pmb{\Delta}}) = \vec{D_{\mathbf{A}}\mathcal{C}(\pmb{\Delta})}.
    \]
    By Theorem~\ref{thm:derivofcholesky}, $D_{\mathbf{A}}\mathcal{C}(\pmb{\Delta}) = \pmb{\Gamma},$ where $\pmb{\Gamma}$ is the solution to the equation
    \[
        \pmb{\Delta} = \pmb{\Gamma} \mathcal{C}(\mathbf{A})^T + \mathcal{C}(\mathbf{A}) \pmb{\Gamma}^T.
    \]
    We can convert this to an equation for $\mathbf{v}_{\pmb{\Gamma}}$ using the fact (Section~10.2.2 of \cite{MatrixCookbook}) that
\begin{equation*}
 \vec{\mathbf{ABC}} = (\mathbf{C}^T\otimes \mathbf{A}) \vec{\mathbf{B}}
\end{equation*}
for arbitrary matrices $\mathbf{A}, \mathbf{B},$ and $\mathbf{C}.$ Specifically, we find that $\mathbf{v}_{\pmb{\Gamma}}$ satisfies
\[
    \mathbf{v}_{\pmb{\Delta}} = (\mathcal{C}(\mathbf{A}) \otimes \mathbf{I})\mathbf{v}_{\pmb{\Gamma}} + (\mathbf{I} \otimes \mathcal{C}(\mathbf{A})) \mathbf{v}_{\pmb{\Gamma}^T}.
\]
Recall that $\pmb{\Delta}$ is symmetric, so $\mathbf{v}_{\pmb{\Gamma}^T} = \mathbf{v}_{\pmb{\Gamma}}$ and we have that
\[
    \mathbf{v}_{\pmb{\Gamma}} = \llbracket \mathcal{C}(\mathbf{A}) \rrbracket^{-1} \mathbf{v}_{\pmb{\Delta}}.
\]
It follows that
\[
    D_{\mathbf{v_A}}C(\mathbf{v}_{\pmb{\Delta}}) = \vec{D_{\mathbf{A}}\mathcal{C}(\pmb{\Delta})} = \vec{\pmb{\Gamma}} = v_{\pmb{\Gamma}} = \llbracket \mathcal{C}(\mathbf{A}) \rrbracket^{-1} \mathbf{v}_{\pmb{\Delta}} = \mathbf{M}^{-1} \mathbf{v}_{\pmb{\Delta}}
\]
as claimed.

To compute the second derivative of $C,$ we use the identity (Section~2.2 of \cite{MatrixCookbook})
\begin{equation}
\label{eqn:derivofinverse}
    D_{\mathbf{x}} A(\mathbf{x})^{-1} = -A(\mathbf{x})^{-1} \cdot D_\mathbf{x} A(\mathbf{x}) \cdot A(\mathbf{x})^{-1}
\end{equation}
that holds for any differentiable matrix-valued function of $\mathbf{x}.$ Using this identity with $\mathbf{x} = \mathbf{v_{\mathbf{A}}}$ and $A = \mathbf{M},$ we see that
\begin{align*}
    D^2_{\mathbf{v_A}}C(\pmb{\delta}_1, \pmb{\delta}_2) = D_{\mathbf{v_A}} ( D_{\mathbf{v_A}}C (\pmb{\delta}_2) ) (\pmb{\delta}_1) & = D_{\mathbf{v_A}} ( \mathbf{M}^{-1} \pmb{\delta}_2 ) (\pmb{\delta}_1) = D_{\mathbf{v_A}} (\mathbf{M}^{-1})(\pmb{\delta}_1) \cdot \pmb{\delta}_2 \\
                                                                             & = -\mathbf{M}^{-1} \cdot D_{\mathbf{v_A}} \mathbf{M} (\pmb{\delta}_1) \cdot \mathbf{M}^{-1} \cdot \pmb{\delta}_2.
\end{align*}
Using the Chain Rule and the linearity of $\llbracket \mathrm{vec}^{-1}(\cdot) \rrbracket,$ we see that
\begin{align}
    D_{\mathbf{v_A}} \mathbf{M} (\pmb{\delta}_1) &= D_{\mathbf{v_A}} \llbracket \mathcal{C}(\mathbf{A}) \rrbracket(\pmb{\delta}_1) = D_{\mathbf{v_A}} \llbracket \mathrm{vec}^{-1}(C(\mathbf{v_A})) \rrbracket (\pmb{\delta}_1) \notag\\
    \label{eqn:derivofinverseofM}
                                                 & = \left\llbracket \mathrm{vec}^{-1}\left(D_{\mathbf{v_A}}C(\mathbf{v_A})(\pmb{\delta}_1)\right)\right\rrbracket = \llbracket \mathrm{vec}^{-1}(\mathbf{M}^{-1} \pmb{\delta}_1) \rrbracket.
\end{align}
Thus, as claimed,
\[
    D_{\mathbf{v_A}}^2 C (\pmb{\delta}_1, \pmb{\delta}_2) = -\mathbf{M}^{-1} \cdot \llbracket \mathrm{vec}^{-1}(\mathbf{M}^{-1} \pmb{\delta}_1) \rrbracket \cdot \mathbf{M}^{-1} \pmb{\delta}_2.
\]

To compute the third derivative of $C,$ we note that the Product Rule (Theorem~2.4.4 of \cite{Marsden}) implies
\[
    D_{\mathbf{x}} [A_1(\mathbf{x}) \cdot A_2(\mathbf{x}) \cdot A_1(\mathbf{x})] = D_{\mathbf{x}} A_1(\mathbf{x}) \cdot A_2(\mathbf{x}) \cdot A_1(\mathbf{x}) + A_1(\mathbf{x}) \cdot D_{\mathbf{x}} A_2(\mathbf{x}) \cdot A_1(\mathbf{x}) + A_1(\mathbf{x}) \cdot A_2(\mathbf{x}) \cdot D_{\mathbf{x}} A_1(\mathbf{x})
\]
for any matrix-valued differentiable functions $M_1$ and $M_2.$ We apply this result with $\mathbf{x} = \mathbf{v_A},$ $A_1 = \mathbf{M}^{-1},$ and $A_2 = \llbracket \mathrm{vec}^{-1}(\mathbf{M}^{-1} \pmb{\delta}_2) \rrbracket$ to see that
\begin{align}
    D_{\mathbf{v_A}}^3 C(\pmb{\delta}_1, \pmb{\delta}_2, \pmb{\delta}_3) = D_{\mathbf{v_A}} ( D_{\mathbf{v_A}}^2 C(\pmb{\delta}_2, \pmb{\delta}_3) ) (\pmb{\delta}_1) & = - D_{\mathbf{v_A}} \mathbf{M}^{-1} (\pmb{\delta}_1) \cdot \llbracket \mathrm{vec}^{-1} (\mathbf{M}^{-1} \pmb{\delta}_2) \rrbracket \cdot \mathbf{M}^{-1} \pmb{\delta}_3 \notag\\
                                                                                               & - \mathbf{M}^{-1} \cdot D_{\mathbf{v_A}} \llbracket \mathrm{vec}^{-1} (\mathbf{M}^{-1} \pmb{\delta}_2) \rrbracket (\pmb{\delta}_1) \cdot \mathbf{M}^{-1} \pmb{\delta}_3 \notag\\
\label{eqn:D3expr}                                                                                               & - \mathbf{M}^{-1} \cdot \llbracket \mathrm{vec}^{-1} (\mathbf{M}^{-1} \pmb{\delta}_2) \rrbracket \cdot D_{\mathbf{v_A}} \mathbf{M}^{-1} (\pmb{\delta}_1) \cdot \pmb{\delta}_3.
\end{align}
From~\eqref{eqn:derivofinverse} and~\eqref{eqn:derivofinverseofM}, we calculate
\[
D_{\mathbf{v_A}}\mathbf{M}^{-1}(\pmb{\delta}_1) = -\mathbf{M}^{-1} \cdot D_{\mathbf{v_A}} \mathbf{M}(\pmb{\delta}_1) \cdot \mathbf{M}^{-1} = -\mathbf{M}^{-1} \cdot \llbracket \mathrm{vec}^{-1}(\mathbf{M}^{-1} \pmb{\delta}_1) \rrbracket \cdot \mathbf{M}^{-1},
\]
and to calculate $D_{\mathbf{v_A}}\llbracket \mathrm{vec}^{-1}(\mathbf{M}^{-1} \pmb{\delta}_2) \rrbracket (\pmb{\delta}_1),$ we use the fact that $\llbracket \mathrm{vec}^{-1}(\cdot)\rrbracket$ is linear:
\begin{align*}
D_{\mathbf{v_A}} \llbracket \mathrm{vec}^{-1}( \mathbf{M}^{-1} \pmb{\delta}_2) \rrbracket (\pmb{\delta}_1) & =  \llbracket \mathrm{vec}^{-1} ( ( D_{\mathbf{v_A}} \mathbf{M}^{-1} )(\pmb{\delta}_1) \cdot \pmb{\delta}_2 ) \rrbracket \\
 &= \llbracket \mathrm{vec}^{-1}( -\mathbf{M}^{-1} \cdot \llbracket\mathrm{vec}^{-1} (\mathbf{M}^{-1} \pmb{\delta}_1 ) \rrbracket \cdot \mathbf{M}^{-1} \cdot \pmb{\delta}_2 ) \rrbracket.
\end{align*}
The claimed expression for $D_{\mathbf{v_A}}^3C$ follows from using these latter two computations to expand~\eqref{eqn:D3expr}.
\end{proof}

Now that we have the first three derivatives of $\mathcal{C},$ we can develop
the second-order Taylor Series expansion of $\mathcal{C}(\mathbf{A})$ and bound
the error of the approximation.
\begin{theorem}
\label{thm:tsapprox}
Assume $\mathbf{A}$ is an positive-definite matrix of order $d+1$ and let 
$\mathbf{v_I} = \vec{\mathbf{I}},$ $\mathbf{M}_s = \llbracket \mathcal{C}(\mathbf{A} + s\mathbf{I}) \rrbracket,$ 
and $\mathbf{E}_s = \llbracket \mathrm{vec}^{-1}(\mathbf{M}_s^{-1} \mathbf{v_I}) \rrbracket.$ 
The second-order Taylor Series approximation to $\mathcal{C}(\mathbf{A} + \lambda \mathbf{I})$ at $\lambda = \lambda_{\textrm{c}}$ is
\[
  \mathbf{p}_{\mathrm{TS}}(\lambda ; \lambda_{\textrm{c}}) = 
    \mathcal{C}(\mathbf{A} + \lambda_{\textrm{c}} \mathbf{I}) + 
    \mathrm{vec}^{-1}\left( (\lambda - \lambda_{\textrm{c}}) \mathbf{M}_{\textrm{c}}^{-1} \mathbf{v_I} - 
      \frac{(\lambda - \lambda_{\textrm{c}})^2}{2} \mathbf{M}_{\textrm{c}}^{-1} \mathbf{E}_{\textrm{c}} \mathbf{M}_{\textrm{c}}^{-1} \mathbf{v_I} \right).
\]
Let $D = (d+1)(d+2)/2;$ then for any $\lambda, \lambda_{\textrm{c}} > 0,$
\[
  \frac{1}{\sqrt{D}} \|\mathcal{C}(\mathbf{A} + \lambda \mathbf{I}) - 
  \mathbf{p}_{\mathrm{TS}}(\lambda ; \lambda_{\textrm{c}})\|_{F} \leq \frac{2|\lambda - \lambda_{\textrm{c}}|^3}{3\sqrt{D}}\mathrm{R}_{[\lambda_{\textrm{c}}, \lambda]}
\]
where
\[
  \mathrm{R}_{[a,b]} := \max_{s \in [\mathrm{min}(a,b), \mathrm{max}(a,b)]} \left( \|\mathbf{M}_s^{-1} \mathbf{E}_s\|_2^2 \|\mathbf{M}_s^{-1} \mathbf{v_I}\|_2 + \|\mathbf{M}_s^{-1}\|_2 \|\mathbf{M}_s^{-1}\mathbf{E}_s\|_2 \|\mathbf{M}_s^{-1} \mathbf{v_I}\|_2^2 \right).
\]
\end{theorem}

To establish this result, we use the following version of Taylor's Theorem (Theorem~2.4.15 of \cite{Marsden}), stated for the case where the first three derivatives of $f$ exist and are continuous.
\begin{theorem}
\label{thm:taylor}
Let $f: \mathbf{X} \rightarrow \mathbf{Y}$ be a three-times continously differentiable mapping. For all $\mathbf{u},\mathbf{h} \in \mathbf{X},$
\[
f(\mathbf{u} + \mathbf{h}) = f(\mathbf{u}) + D^1_\mathbf{u}(\mathbf{h}) + \frac{1}{2}D^2_\mathbf{u}(\mathbf{h},\mathbf{h}) + \mathrm{R}(\mathbf{u},\mathbf{h})(\mathbf{h},\mathbf{h},\mathbf{h}),
\]
where
\[
\mathrm{R}(\mathbf{u}, \mathbf{h}) = \frac{1}{2} \int_0^1 (1-t)^2 \left(D^3_{\mathbf{u} + t\mathbf{h}} - D^3_{\mathbf{u}} \right)\,dt.
\]
\end{theorem}

\begin{proof}[Proof of Theorem~\ref{thm:tsapprox}]
    Let $\mathbf{M}_s = \llbracket \mathcal{C}(\mathbf{A} + s\mathbf{I}) \rrbracket$ and 
    $\mathbf{E}_s = \llbracket \mathrm{vec}^{-1}(\mathbf{M}_s^{-1} \mathbf{v_I}) \rrbracket.$ 
    For convenience, define $\mathbf{M} = \mathbf{M}_{\textrm{c}}$ and $\mathbf{E} = \mathbf{E}_{\textrm{c}}.$
    By Theorem~\ref{thm:derivsofC}, if we take $\mathbf{h} = (\lambda - \lambda_{\textrm{c}}) \mathbf{v_I},$ then
    \begin{align*}
      D_{\mathbf{v_A}}C(\mathbf{h}) & = (\lambda - \lambda_{\textrm{c}}) \mathbf{M}^{-1} \mathbf{v_I}, \\
      D^2_{\mathbf{v_A}}C(\mathbf{h}, \mathbf{h}) & = -(\lambda - \lambda_{\textrm{c}})^2 \mathbf{M}^{-1} \mathbf{E} \mathbf{M}^{-1} \mathbf{v_I} \\
      D^3_{\mathbf{v_A}}C(\mathbf{h}, \mathbf{h}, \mathbf{h}) &= (\lambda - \lambda_{\textrm{c}} )^3 \mathbf{M}^{-1} \Big( 2 \mathbf{E} \mathbf{M}^{-1} \mathbf{E}
    + \llbracket \mathrm{vec}^{-1}(\mathbf{M}^{-1} \mathbf{E} \mathbf{M}^{-1} \mathbf{v_I}) \rrbracket \Big) \mathbf{M}^{-1} \mathbf{v_I},
    \end{align*}
    so the application of Taylor's Theorem to $C$ at $\mathbf{v_A + \lambda_{\textrm{c}} \mathbf{v_I} }$ gives the expansion
    \begin{multline*}
      C(\mathbf{v_A} + \lambda \mathbf{v_I}) = C(\mathbf{v_A} + \lambda_{\textrm{c}} \mathbf{v_I}) + (\lambda - \lambda_{\textrm{c}}) \mathbf{M}^{-1} \mathbf{v_I}
      - \frac{(\lambda - \lambda_{\textrm{c}} )^2}{2} \mathbf{M}^{-1} \mathbf{E} \mathbf{M}^{-1}\mathbf{v_I}  \\
      + (\lambda - \lambda_{\textrm{c}})^3 \mathrm{R}(\mathbf{v_A} + \lambda_\textrm{c} \mathbf{v_I}, 
          (\lambda - \lambda_{\textrm{c}}) \mathbf{v_I})(\mathbf{v_I}, \mathbf{v_I}, \mathbf{v_I}).
    \end{multline*}
    Since $\vec{\cdot}$ is an isometry, \emph{i.e.,} $\|\vec{\mathbf{X}}\|_2 = \|\mathbf{X}\|_F$ for any matrix $\mathbf{X}$, 
    we conclude that the Taylor expansion of the Cholesky factorization map around $\mathbf{A} + \lambda_{\textrm{c}} \mathbf{I}$ is given by
    \begin{multline}
    \label{eqn:cholts}
    \mathcal{C}(\mathbf{A} + \lambda \mathbf{I}) = \mathcal{C}(\mathbf{A} + \lambda_{\textrm{c}} \mathbf{I}) +
    \mathrm{vec}^{-1}\left( (\lambda - \lambda_{\textrm{c}}) \mathbf{M}^{-1} \mathbf{v_I} 
      - \frac{(\lambda - \lambda_{\textrm{c}})^2}{2} \mathbf{M}^{-1} \mathbf{E} \mathbf{M}^{-1} \mathbf{v_I} \right)   \\
    + (\lambda - \lambda_{\textrm{c}})^3 \mathrm{vec}^{-1}(\mathrm{R}(\mathbf{v_A} + \lambda_{\textrm{c}}, 
    (\lambda - \lambda_{\textrm{c}}) \mathbf{v_I}) (\mathbf{v_I}, \mathbf{v_I}, \mathbf{v_I})).
    \end{multline}
     The remainder term can be bounded as follows:
    \begin{align*}
      \|\mathrm{R}(\mathbf{v_A} + \lambda_{\textrm{c}}, (\lambda - \lambda_{\textrm{c}}) \mathbf{v_I}) (\mathbf{v_I}, \mathbf{v_I}, \mathbf{v_I})\|_2 & = 
      \left\| \int_0^1 \frac{(1 - t)^2}{2} (D^3_{\mathbf{v_A} + [(1-t)\lambda_{\textrm{c}} + t\lambda] \mathbf{v_I}}C 
        - D_{\mathbf{v_A} + \lambda_{\textrm{c}}\mathbf{v_I}}^3C)( \mathbf{v_I}, \mathbf{v_I}, \mathbf{v_I}) \,dt \right\|_2 \\
      & \leq \frac{1}{6} \max_{t \in [0,1]} \big\|\big(D^3_{\mathbf{v_A} + [(1-t)\lambda_{\textrm{c}} + t \lambda] \mathbf{v_I}}C 
      - D_{\mathbf{v_A} + \lambda_{\textrm{c}} \mathbf{v_I}}^3C\big)(\mathbf{v_I},\mathbf{v_I},\mathbf{v_I})\big\|_2 \\
      & \leq \frac{1}{3} \max_{s \in [\lambda_{\textrm{c}}, \lambda]} 
        \|D^3_{\mathbf{v_A} + s \mathbf{v_I}} C(\mathbf{v_I}, \mathbf{v_I}, \mathbf{v_I})\|_2 \\
        & \leq \frac{1}{3} \max_{s \in [\lambda_{\textrm{c}}, \lambda]} 
        \left( 2 \|\mathbf{M}_s^{-1} \mathbf{E}_s \mathbf{M}_s^{-1} \mathbf{E}_s \mathbf{M}_s^{-1}\mathbf{v_I}\|_2 \right.\\
        & \hspace{4em} \left. + \|\mathbf{M}_s^{-1} \llbracket \mathrm{vec}^{-1} (\mathbf{M}_s^{-1} \mathbf{E}_s \mathbf{M}_s^{-1} \mathbf{v_I}) \rrbracket
          \mathbf{M}_s^{-1} \mathbf{v_I}\|_2 \right) \\
        & \leq \frac{1}{3} \max_{s \in [\lambda_{\textrm{c}}, \lambda]} 
        \left( 2 \|\mathbf{M}_s^{-1} \mathbf{E}_s\|_2^2 \|\mathbf{M}_s^{-1} \mathbf{v_I}\|_2 \right. \\
        &  \hspace{4em} \left. + \|\mathbf{M}_s^{-1} \|_2 \|\llbracket \mathrm{vec}^{-1}(\mathbf{M}_s^{-1} \mathbf{E}_s \mathbf{M}_s^{-1} \mathbf{v_I} ) \rrbracket\|_2 
          \|\mathbf{M}_s^{-1} \mathbf{v_I}\|_2 \right).
    \end{align*}
    To further simplify this estimate, note that for any matrix $\mathbf{X},$
   \begin{align*}
   \|\llbracket \mathrm{vec}^{-1}(\mathbf{v_X}) \rrbracket\|_2 & = \|\llbracket\mathbf{X}\rrbracket\|_2 = \|\mathbf{I} \otimes \mathbf{X} + \mathbf{X} \otimes \mathbf{I}\|_2 \\
   & \leq \|\mathbf{I} \otimes \mathbf{X}\|_2 + \|\mathbf{X} \otimes \mathbf{I}\|_2 \\
   & \leq 2 \|\mathbf{I}\|_2 \|\mathbf{X}\|_2 \leq 2 \|\mathbf{X}\|_F \\
   & = 2 \|\mathbf{v_X}\|_2.
   \end{align*}
   In particular,
   \[
    \|\llbracket \mathrm{vec}^{-1}( \mathbf{M}_s^{-1} \mathbf{E}_s \mathbf{M}_s^{-1} \mathbf{v_I}) \rrbracket\|_2 \leq 2 \|\mathbf{M}_s^{-1} \mathbf{E}_s \mathbf{M}_s^{-1} \mathbf{v_I}\|_2 \leq 2 \|\mathbf{M_s}^{-1} \mathbf{E}_s\|_2 \|\mathbf{M}_s^{-1} \mathbf{v_I}\|_2.
   \]
   It follows that
   \begin{multline*}
     \|\mathrm{R}(\mathbf{v_A} + \lambda_{\textrm{c}}, (\lambda - \lambda_{\textrm{c}}) \mathbf{v_I}) (\mathbf{v_I}, \mathbf{v_I}, \mathbf{v_I})\|_2\leq
     \frac{2}{3} \max_{s \in [\lambda_{\textrm{c}}, \lambda]} 
     \left( \|\mathbf{M}_s^{-1} \mathbf{E}_s\|_2^2 \|\mathbf{M}_s^{-1} \mathbf{v_I}\|_2  \right. \\
     \left. + \|\mathbf{M}_s^{-1} \|_2 \|\mathbf{M}_s^{-1} \mathbf{E}_s\|_2 \|\mathbf{M}_s^{-1} \mathbf{v_I}\|_2^2 \right),
   \end{multline*}
   where for convenience we use $[a,b]$ to denote $[\mathrm{min}(a,b), \mathrm{max}(a,b)].$
   As a consequence of this bound and~\eqref{eqn:cholts}, we conclude that
   \begin{multline*}
     \frac{1}{\sqrt{D}}\|\mathcal{C}(\mathbf{A} + \lambda \mathbf{I}) - \mathbf{p}_{\mathrm{TS}}(\lambda; \lambda_{\textrm{c}})\|_{F}  \\
     \leq \frac{2|\lambda - \lambda_{\textrm{c}}|^3}{3\sqrt{D}} \max_{s \in [\lambda_{\textrm{c}}, \lambda]} 
     \left( \|\mathbf{M}_s^{-1} \mathbf{E}_s\|_2^2 \|\mathbf{M}_s^{-1} \mathbf{v_I}\|_2 + 
       \|\mathbf{M}_s^{-1} \|_2 \|\mathbf{M}_s^{-1} \mathbf{E}_s\|_2 \|\mathbf{M}_s^{-1} \mathbf{v_I}\|_2^2 \right).
 \end{multline*}

\end{proof}

Our next result quantifies the distance between $\mathbf{p}_{\mathrm{TS}}$ and
$\mathbf{p}_\pi$, the polynomial approximation fit using the
\textit{pi}Cholesky interpolation procedure. The result is based on the
observation that $\mathbf{p}_{\mathrm{TS}}$ can be recovered using the same
algorithm that returns the \textit{pi}Cholesky polynomial $\mathbf{p}_\pi,$ if
samples from $\mathbf{p}_{\mathrm{TS}}$ are used for the interpolation
instead of samples from $\mathcal{C}(\mathbf{A} + \lambda \mathbf{I}).$ Thus
the error $\|\mathbf{p}_{\mathrm{TS}} - \mathbf{p}_{\pi}\|_F$ can be
interpreted as being caused by sampling error, and bounded using results on the
stability of least squares systems.

\begin{theorem}
 \label{thm:interpstability}
 Let $\mathbf{A}$ and $\mathbf{p}_{\mathrm{TS}}$ be as in
 Theorem~\ref{thm:tsapprox} and let $\mathbf{p}_\pi$ be the matrix whose entries
 are the second-order polynomial approximations to the entries of
 $\mathcal{C}(\mathbf{A} + \lambda \mathbf{I})$ with coefficients defined using
 Algorithm~\ref{alg:chol_interp}. Assume that the $g$ sampled regularization points $\lambda_i$ used in Algorithm~\ref{alg:chol_interp} 
 all lie within distance $w$ of $\lambda_{\textrm{c}}.$ With $\mathbf{V}$ and $D$ as in
 Algorithm~\ref{alg:chol_interp} and $\mathrm{R}_{[a,b]}$ defined as in Theorem~\ref{thm:tsapprox},
 \[
   \frac{1}{\sqrt{D}}\|\mathbf{p}_{\mathrm{TS}}(\lambda; \lambda_{\textrm{c}}) - \mathbf{p}_\pi(\lambda)\|_F \leq 
   \sqrt{\frac{g}{D}} w^3 [1 + (\lambda - \lambda_{\textrm{c}})^2](\lambda_{\textrm{c}} + 1) 
  \|\mathbf{V}^\dagger\|_2 \mathrm{R}_{[\lambda_{\textrm{c}} - w, \lambda_{\textrm{c}}+w]}.
 \]
 for any $\lambda > 0.$
\end{theorem}

\begin{proof} Let $\mathbf{T}_{\pi}$ and $\mathbf{V}_{\pi}$ denote the matrices
  $\mathbf{T}$ and $\mathbf{V}$ in Algorithm~\ref{alg:chol_interp}, and
  $\pmb{\Theta}_{\pi}$ denote the matrix of coefficients fit by the
  \textit{pi}Cholesky algorithm, so $\pmb{\Theta}_{\pi} = \mathbf{V}_{\pi}^\dagger \mathbf{V}_{\pi}.$
  Let $\pmb{\tau}_{\pi} = [1\,\,\lambda\,\,\lambda^2]^T$. By construction, the  
  \textit{pi}Cholesky approximation to the Cholesky factor of $\mathbf{A} + \lambda \mathbf{I}$ is given by
  $\vec{\mathbf{p}_{\pi}(\lambda)} = \pmb{\tau}_{\pi}^T \pmb{\Theta}_{\pi}.$

Analogously, let $\mathbf{T}_{\mathrm{TS}}$ denote the $g \times D$ matrix
constructed by sampling $\vec{\mathbf{p}_{\mathbf{TS}}(\cdot;
  \lambda_{\textrm{c}})}$ at the $g$ values $\lambda_1, \ldots, \lambda_g$ and let
$\mathbf{V}_{\mathrm{TS}}$ denote the $g \times 3$ matrix with rows consisting
of the vectors $[1\,\, (\lambda_i - \lambda_{\textrm{c}})\,\, (\lambda_i -
\lambda_{\textrm{c}})^2]$ for $i=1, \ldots, g.$ Observe that since the entries
of $\mathbf{p}_{\mathbf{TS}}(\cdot; \lambda_{\textrm{c}})$ are quadratic
polynomials, because $g > 3$ the relationship $ \mathbf{T}_{\mathrm{TS}} =
\mathbf{V}_{\mathrm{TS}} \pmb{\Theta}_{\mathrm{TS}}$ holds for
$\pmb{\Theta}_{\mathrm{TS}} = \mathbf{V}_{\mathrm{TS}}^\dagger
\mathbf{T}_{\mathrm{TS}}.$ Let $\pmb{\tau}_{\mathrm{TS}} = [1\,\,(\lambda - \lambda_{\textrm{c}})\,\,(\lambda - \lambda_{\textrm{c}})^2],$ then 
$\vec{\mathbf{p}_{\mathrm{TS}}(\lambda; \lambda_c)} = \pmb{\tau}_{\mathrm{TS}}^T \pmb{\Theta}_{\mathrm{TS}}.$

Simple calculations verify that $\mathbf{V}_{\mathrm{TS}} = \mathbf{V}_{\pi} \mathbf{M}$ and $\pmb{\tau}_{\mathrm{TS}} = \mathbf{M}^T \pmb{\tau}_{\pi},$ where
\[
  \mathbf{M} = \begin{bmatrix} 
    1 & -\lambda_{\textrm{c}} & \lambda_{\textrm{c}}^2 \\
    0 & 1 & -\lambda_{\textrm{c}} \\
    0 & 0 & 1
  \end{bmatrix}.
\]
Consequently,
\begin{align*}
  \frac{1}{D}\|\mathbf{p}_{\mathrm{TS}}(\lambda; \lambda_{\textrm{c}}) - \mathbf{p}_\pi(\lambda)\|_F & =
  \frac{1}{D}\big\|\pmb{\tau}_{\mathrm{TS}}^T\pmb{\Theta}_{\mathrm{TS}} - \pmb{\tau}_{\pi}^T \pmb{\Theta}_{\pi}\big\|_2   \\
  & = \frac{1}{D} \big\|\pmb{\tau}_{\pi}^T\mathbf{M} \mathbf{V}_{\mathrm{TS}}^\dagger \mathbf{T}_{\mathrm{TS}} 
  - \pmb{\tau}_{\pi}^T \mathbf{V}_{\pi}^\dagger \mathbf{T}_{\pi}\big\|_2
\end{align*}
We expand $\mathbf{V}_{\mathrm{TS}}^\dagger$ by noting that $\mathbf{V}_{\pi}$ has full column rank (because
the $g$ values of $\lambda_i$ are unique, $g > 3$, and $\mathbf{V}_{\pi}$
consists of the first 3 columns of a $g \times g$ Vandermonde matrix) and $\mathbf{M}$ has full row rank (in fact, it is invertible). It follows that
$\mathbf{V}_{\mathrm{TS}}^\dagger = \big(\mathbf{V}_{\pi} \mathbf{M})^\dagger = \mathbf{M}^{-1} \mathbf{V}_{\pi}^\dagger.$
Accordingly, we find that
\begin{align*}
  \frac{1}{D}\|\mathbf{p}_{\mathrm{TS}}(\lambda; \lambda_{\textrm{c}}) - \mathbf{p}_\pi(\lambda)\|_F & =
  \frac{1}{D}\big\|\pmb{\tau}_{\pi}^T \mathbf{V}_{\pi}^\dagger \mathbf{T}_{\mathrm{TS}} - \pmb{\tau}_{\pi}^T \mathbf{V}_{\pi}^\dagger \mathbf{T}_{\pi}\big\|_2 \\
  & \leq \frac{1}{D} \|\pmb{\tau}_{\pi}\|_2 \|\mathbf{V}_{\pi}^\dagger\|_2 \|\mathbf{T}_{\mathrm{TS}} - \mathbf{T}_{\pi}\|_2 \\
  & \leq \frac{1}{D} \|\mathbf{M}^{-1}\|_2 \|\pmb{\tau}_{\mathrm{TS}}\|_2 \|\mathbf{V}_{\pi}^\dagger\|_2 \|\mathbf{T}_{\mathrm{TS}} - \mathbf{T}_{\pi}\|_2.
\end{align*}
We have 
\[
  \|\pmb{\tau}_{\mathrm{TS}}\|_2 = \sqrt{1 + (\lambda - \lambda_{\textrm{c}})^2 + (\lambda - \lambda_{\textrm{c}})^4} 
  \leq 1 + (\lambda - \lambda_{\textrm{c}})^2. 
\]
To estimate $\|\mathbf{M}^{-1}\|_2$, observe first that 
\[
  \mathbf{M}^{-1} = \begin{bmatrix}
    1 & \lambda_{\textrm{c}} & 0 \\
    0 & 1 & \lambda_{\textrm{c}} \\
    0 & 0 & 1
  \end{bmatrix};
\]
the simple bound $\|\mathbf{M}^{-1}\|_2 \leq \sqrt{1 + \lambda_{\textrm{c}}^2 + 2\lambda_{\textrm{c}}} = \lambda_{\textrm{c}} + 1$ 
follows by considering the action of $\mathbf{M}^{-1}$ on unit length vectors. The matrices 
$\mathbf{T}_{\mathrm{TS}}$ and $\mathbf{T}_\pi$ have dimension $g \times D$, with rows comprising vectorized samples from 
$\mathbf{p}_{\mathrm{TS}}(\lambda; \lambda_{\textrm{c}})$ and $\mathcal{C}(\mathbf{A} - \lambda \mathbf{I})$ respectively. 
Theorem~\ref{thm:tsapprox} bounds the root mean squared error between the rows of the two sample matrices corresponding to the same value of $\lambda,$ giving
\[
  \frac{1}{D} \|\mathbf{T}_{\mathrm{TS}} - \mathbf{T}_{\pi}\|_2 \leq 
  \frac{\sqrt{g}}{D} \max_{i} \| \mathbf{p}_{\mathrm{TS}}(\lambda_i; \lambda_{\textrm{c}}) - \mathcal{C}(\mathbf{A} + \lambda_i \mathbf{I})\|_F \leq
  \frac{2 \sqrt{g}}{3 D} \max_i |\lambda_i - \lambda_{\textrm{c}}|^3 \mathrm{R}_{\lambda_i},
\]
where $\mathrm{R}_{\lambda_i}$ is as defined in Theorem~\ref{thm:tsapprox}.

Putting the pieces together, we conclude that 
\[
  \frac{1}{D}\|\mathbf{p}_{\mathrm{TS}}(\lambda; \lambda_{\textrm{c}}) - \mathbf{p}_\pi(\lambda)\|_F \leq 
  [1 + (\lambda - \lambda_{\textrm{c}})^2]\left( \frac{2(\lambda_{\textrm{c}} + 1)\sqrt{g} }{3D} \right) 
  \|\mathbf{V}_{\pi}^\dagger\|_2 \max_{i=1,\ldots,g} |\lambda_i - \lambda_{\textrm{c}}|^3 \mathrm{R}_{\lambda_i}.
\]

\end{proof}

Our guarantee on the performance of the \textit{pi}Cholesky procedure now
follows from Theorems~\ref{thm:tsapprox} and~\ref{thm:interpstability} and the
triangle inequality.
\begin{theorem}
 \label{thm:mainresult}
 Assume $\mathbf{A}$ is an positive-definite matrix of order $d+1,$ and let $D
 = (d+1)(d+2)/2.$ Given $\lambda_{c} > \gamma  \geq w > 0,$ assume that
 Algorithm~\ref{alg:chol_interp} is used with $g$ samples of $\lambda$ from
 $[\lambda_{\textrm{c}} - w, \lambda_{\textrm{c}} + w]$ to approximate
 $\mathcal{C}(\mathbf{A} + \lambda \mathbf{I}).$ The error of the approximation
 over the interval $[\lambda_{\textrm{c}} - \gamma, \lambda_{\textrm{c}} +
 \gamma]$ is uniformly bounded by
 \[
   \frac{1}{\sqrt{D}} \| \mathcal{C}(\mathbf{A} + \lambda \mathbf{I}) - 
   \mathbf{p}_{\pi}(\lambda) \|_F \leq \big[ \gamma^3 + 
   \sqrt{g} w^3 (1 + \gamma^2) (\lambda_{\textrm{c}} + 1) \|\mathbf{V}^\dagger\|_2 \big]
   \frac{\mathrm{R}_{[\lambda_{\textrm{c}} - \gamma, \lambda_{\textrm{c}} + \gamma]}}{\sqrt{D}}
 \]
 Here, $\mathbf{V}$ is defined in Algorithm~\ref{alg:chol_interp} and 
 $\mathrm{R}_{[a,b]}$ is defined in Theorem~\ref{thm:tsapprox}.
\end{theorem}

\begin{proof}
    Applying the triangle inequality and Theorems~\ref{thm:tsapprox} 
    and~\ref{thm:interpstability}, for any $\lambda \in [\lambda_{\textrm{c}} - \gamma, \lambda_{\textrm{c}} + \gamma],$
\begin{align*}
  \frac{1}{\sqrt{D}}\|\mathcal{C}(\mathbf{A} + \lambda \mathbf{I}) - \mathbf{p}_\pi(\lambda)\|_F 
  & \leq \frac{1}{\sqrt{D}} \|\mathcal{C}(\mathbf{A} + \lambda \mathbf{I}) 
  - \mathbf{p}_{\mathrm{TS}}(\lambda; \lambda_{\textrm{c}})\|_F 
  + \frac{1}{\sqrt{D}}\|\mathbf{p}_{TS}(\lambda; \lambda_{\textrm{c}}) - \mathbf{p}_\pi(\lambda)\|_F \\
  & \leq \frac{\gamma^3}{\sqrt{D}} \mathbf{R}_{[\lambda_{\textrm{c}}, \lambda]} 
  + \sqrt{\frac{g}{D}} w^3 (1 + \gamma^2)(\lambda_{\textrm{c}} + 1)
  \|\mathbf{V}^\dagger\|_2 \mathrm{R}_{[\lambda_{\textrm{c}} - w, \lambda_{\textrm{c}}+w]} \\
  & \leq \big[\gamma^3 + \sqrt{g} w^3 ( 1 + \gamma^2) (\lambda_{\textrm{c}} + 1)\|\mathbf{V}^\dagger\|_2 \big]
  \frac{\mathrm{R}_{[\lambda_{\textrm{c}} - w, \lambda_{\textrm{c}}+w]}}{\sqrt{D}}.
\end{align*}
\end{proof}

\section{Vectorizing a Cholesky Factor}
\label{sec:ichallenges}
\noindent We now turn our attention towards solving the efficiency challenges of Algorithm~\ref{alg:chol_interp}. Recall that in Algorithm~\ref{alg:chol_interp}, the operations $\mathbf{G}=\mathbf{V}^{\mathrm{T}}{\mathbf{T}}$ and $\Theta=\mathbf{H}_{\lambda}^{-1}\mathbf{G}_{\lambda}$ can be done efficiently by employing BLAS-$3$ level matrix-matrix computations. However, this would require having the Cholesky factor $\mathbf{L}$ for each $\lambda$ value as one of the rows of the target matrix $\mathbf{T}$. This matrix-vector conversion can be a computational bottleneck if done naively. For instance, concatenating the lower-triangular part of $\mathbf{L}$ in a \textbf{row-wise} manner would result in significant number of non-contiguous memory copies. On the other hand, vectorizing $\mathbf{L}$ as a \textbf{full-matrix} would increase the number of interpolations (ln. $5$-$6$ in Algorithm~\ref{alg:chol_interp}) by a factor of $2$.

\begin{table*}
\small
\centering
\begin{tabular}{|c||c|c|c|c||c|c|c|c||c|c|c|c|}
\hline
\multirow{2}{*}{\textbf{Dimensions}} & \multicolumn{4}{c||}{\textbf{Row-wise}} & \multicolumn{4}{c||}{\textbf{Full-matrix}} & \multicolumn{4}{c|}{\textbf{Recursive}} \\
\cline{2-13} & Vec & Fit & Interp & \textbf{Total} & Vec & Fit & Interp & \textbf{Total} & Vec & Fit & Interp & \textbf{Total} \\
\hline
$1024$ & 1.48 & 0.33 & 1.25 & \textbf{3.06} & 0.15 & 0.72 & 2.47 & \textbf{3.34} & 0.79 & 0.33 & 1.27 & \textbf{2.39} \\
\hline
$2048$ & 5.94 & 1.49 & 4.86 & \textbf{12.29} & 0.97 & 3.03 & 15.79 & \textbf{19.79} & 2.21 & 1.49 & 4.95 & \textbf{8.65} \\
\hline
$4096$ & 29.85 & 5.79 & 27.46 & \textbf{63.1} & 3.76 & 11.46 & 59.9 & \textbf{75.13} & 6.66 & 5.85 & 27.64 & \textbf{40.15} \\
\hline
$8192$ & 129.6 & 23.31 & 140.8 & \textbf{293.7} & 18.29 & 49.49 & 283.1 & \textbf{350.9} & 22.31 & 22.53 & 107.4 & \textbf{152.3} \\
\hline
$16384$ & 500.9 & 98.83 & 515.76 & \textbf{1116} & 69.4 & 209.8 & 1041.8 & \textbf{1321.1} & 91.74 & 99.36 & 512.8 & \textbf{703.9} \\
\hline
\end{tabular}
\caption{Comparison of the timing results of $pi$Cholesky before and after optimization, using MNIST data. The measured times (in seconds) include the tranformation between an upper-triangular Cholesky factor and its vectorized form, as well as fitting and interpolating the polynomial functions, abbreviated as ``vec'', ``fit'', and ``interp'' respectively.}
\label{table:timingbreakup}
\end{table*}

We now present an efficient way to vectorize a Cholesky factor $\mathbf{L}$ such that: (i) we achieve aligned memory copy and, (ii) we have non-redundant computation in the interpolation step of Algorithm~\ref{alg:chol_interp}.

Let $h=d+1$ denote the dimension of $\mathbf{L}$. Also, without the loss of generality, let us consider $h$ to be a power of two. We use a divide-and-conquer strategy to partition the lower-triangular part into a square matrix and two smaller lower-triangular matrices, \textit{i.e.},
\begin{equation}
\begin{small}
\label{eq:recursive}
\begin{aligned}
& \mathbf{L}^{12} = \mathbf{L}({h \over 2}+1:h, 1:{h \over 2}) \\
& \mathbf{L}^{11} = \mathbf{L}(1:{h \over 2}, 1:{h \over 2})\\
& \mathbf{L}^{22} = \mathbf{L}({h \over 2}+1:h, {h \over 2}+1:h)
\end{aligned}
\end{small}
\end{equation}
\noindent This strategy is depicted in Figure~\ref{fig:recursive}(a). The vectorization of $\mathbf{L}$ is the concatenation of the vectorizations of $\mathbf{L}^{12}$, $\mathbf{L}^{11}$, and $\mathbf{L}^{22}$. To vectorize the square matrix $\mathbf{L}^{12}$, we can simply use the ordering in the full-matrix strategy because $\mathbf{L}^{12}$ has no special structure and is already memory-aligned. For the smaller lower-triangular matrices $\mathbf{L}^{11}$ and $\mathbf{L}^{22}$, we recursively partition its lower-triangular part using the partitioning scheme in Equation~\ref{eq:recursive} until a threshold dimension $h_0$ is reached. At the deepest level of the recursion, we use the row-wise strategy to vectorize the $h_0 \times h_0$ matrix, which for a sufficiently small $h_0$ is not expensive. The resulting \textbf{recursive} partitioning of the Cholesky factor $\mathbf{L}$ is depicted in Figure~\ref{fig:recursive}(b).

\begin{figure}
\centering
\subfloat[] {\includegraphics[width=0.2\textwidth]{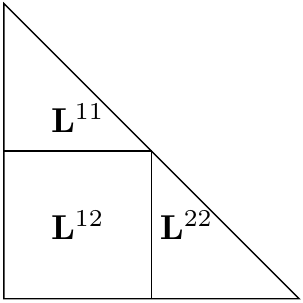}}
\hspace{0.5in}
\subfloat[] {\includegraphics[width=0.2\textwidth]{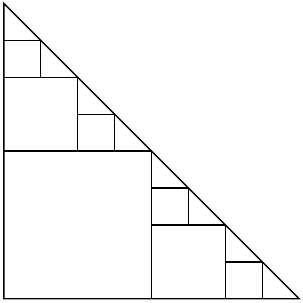}}
  \caption{{\small Proposed \textbf{recursive} strategy to vectorize $\mathbf{L}$. (a) Partitioning of lower-triangular part. (b) Final partitioning obtained by recursively applying the scheme in (a).}}
  \label{fig:recursive}
  \vspace{-0.3cm}
\end{figure}

Note that our recursive vectorization strategy can be applied to store any upper/lower-triangular matrix and is more generally applicable beyond the scope of this paper. Table~\ref{table:timingbreakup} gives an empirical sense of the efficiency of our recursive strategy compared to row-wise and full-matrix ones.

\section{Experiments}
\label{sec:exps}
\vspace{0.1cm}\noindent We now present the timing performance and approximation error of our \textit{pi}Cholesky framework using multiple data-sets, and compare it with standard and state-of-the-art methods. Our results show that our proposed scheme is able to accelerate large-scale linear regression substantially and achieve high accuracy in selecting the optimal regularization parameter.

\subsection{Data Sets}

\vspace{0.1cm}\noindent We use four image data sets in our experiments. Some of the important details of these data-sets are given in Table \ref{table:datasets}. For MNIST and COIL-$100$ data-sets, we projected the samples to $1023$, $2047$, $4095$, $8191$, and $16383$ dimensions using the randomized polynomial kernel~\cite{kar2012random}. For the Caltech-$101$ and $256$ data-sets, we projected the samples to $16383$ dimensions using the spatial pyramid framework~\cite{lazebnik2006beyond}. All data-sets were converted to $2$ class problems with equal numbers of positive and negative samples. In the following, we denote $h=d+1$ as the projected plus intercept dimensions.


\begin{table}[!t]
\small
\centering
\begin{tabular}{|c|c|c|}
\hline
& Dimensionality & \# Samples \\
\hline
\textbf{MNIST} & $28 \times 28$ & 60,000 \\
\hline
\textbf{COIL-100} & $28 \times 28$ & 7,200 \\
\hline
\textbf{Caltech-101} & $320 \times 200$ & 8,677 \\
\hline
\textbf{Caltech-256} & $320 \times 200$ & 29,780 \\
\hline
\end{tabular}
\caption{Summary of the data-sets used -- MNIST~\cite{lecunmnist}, COIL-$100$~\cite{coil100}, Caltech-$101$~\cite{fei2007learning}, and Caltech-$256$~\cite{griffin2007caltech}.}
\label{table:datasets}
\end{table}

\subsection{Comparative Algorithms}
\label{exp:algs}

\noindent We compare the following algorithms for solving least squares with cross validation:\vspace{-0.15cm}

\begin{enumerate}
\item \textbf{Exact Cholesky} (\texttt{Chol}) -- Apply Cholesky factorization to the Hessian matrix $\mathbf{H}$ for each candidate $\lambda$ value as described in $\S$\ref{subsec:exactcholesky}.

\item \textbf{\textit{pi}Cholesky} (\texttt{PIChol}) -- The proposed approach.

\item \textbf{Multi-level Cholesky} ({\texttt{MChol}}) -- We consider a binary-search-like multi-level approach that progressively narrows down the search range of the optimal $\lambda$. More rigorously, starting with an initial range $[10^{c-s}, 10^{c+s}]$ where $c,s \in \mathbb{R}$ and $s > 0$, we perform the following three steps iteratively:
\begin{enumerate}
\item Evaluate the hold-out errors $h(\lambda)$ by computing the exact Cholesky factorization at $\lambda=10^{c-s}, 10^c, 10^{c+s}$.
\item Choose the $\lambda$ value with the smallest hold-out error in Step (a), \textit{i.e.}, $\lambda_{\text{opt}} = \arg\min_{\lambda} h(\lambda)$.
\item Update $c,s$: $c \leftarrow \log \lambda_{\text{opt}}$, $s \leftarrow s/2$, and define the updated range $[10^{c-s}, 10^{c+s}]$.
\end{enumerate}
The procedure ends when $s \leq s_0$ for a given value $s_0 > 0$. We used this approach to find the initial search ranges for each data-set. These ranges are then used by all the comparative algorithms (including \texttt{MChol}) to find the optimal $\lambda$ value.

\item \textbf{Exact SVD} (\texttt{SVD}) -- SVD is a standard method for solving ridge regression~\cite{svdregression}. Given an $n \times (d+1)$ design matrix $\mathbf{X}$ and its SVD $\mathbf{X} = \mathbf{U} \mathbf{\Sigma} \mathbf{V}^{\textrm{T}}$, the solution of the coefficient vector $\theta$ can be derived from Equation \ref{eq:reg_normal_equation_Hg}:
\begin{equation}
\theta = \mathbf{V} \text{diag}\left({\sigma_1 \over \sigma_1^2 + \lambda}, \cdots, {\sigma_{d+1} \over \sigma_{d+1}^2 + \lambda}\right) \mathbf{U}^{\mathrm{T}} \mathbf{g}
\label{eq:reg_normal_equation_svd}
\end{equation}
where $\sigma_1, \cdots, \sigma_{d+1}$ are the singular values of $\mathbf{X}$ in non-increasing order.

\item \textbf{Truncated SVD} (\texttt{t-SVD}) -- Instead of using the full SVD, we compute the $k$ singular vectors $\hat{\mathbf{U}}$ and $\hat{\mathbf{V}}$ that correspond to the $k$ largest singular values of $\mathbf{X}$, so that $\mathbf{X}$ is approximated by $\mathbf{X} \approx \hat{\mathbf{U}} \hat{\mathbf{\Sigma}} \hat{\mathbf{V}}^{\textrm{T}}$, where $\hat{\mathbf{U}} \in \mathbb{R}^{n \times k}$, $\hat{\mathbf{\Sigma}} \in \mathbb{R}^{k \times k}$, $\hat{\mathbf{V}} \in \mathbb{R}^{(d+1) \times k}$. Then we obtain the coefficient vector $\theta$ accordingly. We used an iterative solver to compute the truncated SVD which is faster than the algorithm for computing the full SVD.

\item \textbf{Randomized Approximate SVD} (\texttt{r-SVD}) -- Random projections have been shown to approximately solve the truncated SVD problem efficiently. In this work, we use the randomized SVD algorithm described in~\cite{halko2011finding}.
\end{enumerate}


\vspace{0.1cm}\noindent Recall that while QR decomposition is another feasible algorithm to solve the least squares problem~\cite{golub2012matrix}, it is generally applied on the design matrix $\textbf{X}$, and cannot be easily applied to linear regression with regularization. QR decomposition can also solve the linear system represented by the Hessian matrix directly as well, but it is more expensive than Cholesky factorization which exploits the symmetric and positive-definite property of the Hessian. Therefore, in our analysis we do not include the comparison with QR decomposition.

\subsection{Experiment Settings}

\vspace{0.1cm}\noindent For \texttt{Chol}, \texttt{SVD}, \texttt{t-SVD}, and \texttt{r-SVD}, we search for the optimal $\lambda$ value from a candidate set of 31 exponentially spaced $\lambda$ values. For \texttt{PIChol}, we sparsely sample 4 $\lambda$ values from those 31 values and interpolate Cholesky factors using second-order polynomial functions, \textit{i.e.}, $g=4$ and $r=2$ in Algorithm \ref{alg:chol_interp}. Second-order polynomial functions are appropriate since we empirically find that the Cholesky entries are typically concave (Figure~\ref{fig:exact_inter_factors}) and the hold-out error curves are typically convex (Figures~\ref{fig:holdout_zoomin_mnist},\ref{fig:holdout_zoomin_coil_caltech}) for all the data sets. For \texttt{MChol} described in Section $\S$\ref{exp:algs}, we set the parameters $s=1.5$ and $s_0=0.0025$. For all the six algorithms, the range from which candidate $\lambda$ values are drawn is set to $[10^{-3}, 1]$, $[10^{-3}, 1]$, $[10^{-8}, 10^{-5}]$, and $[10^{-3}, 1]$ for the four data sets respectively. We ran all our experiments on an $8$-core machine using multi-threaded linear algebra routines.

\subsection{Timing Results}

\begin{figure}
\centering
\includegraphics[width=0.5\textwidth]{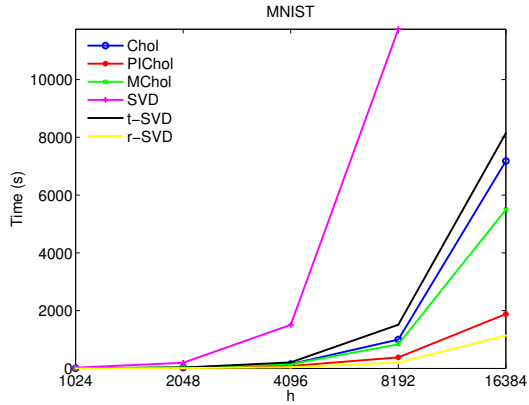}
\caption{Times (in seconds) of the $6$ considered algorithm, summed over all cross validation folds, as a function of $h$ on MNIST data.}
\label{fig:timing}
\end{figure}

\begin{table}
\small
\centering
\begin{tabular}{|c||c||c||c||c|c|}
\hline
\multirow{2}{*}{} & \multirow{2}{*}{\textbf{MNIST}} & \textbf{COIL} & \textbf{Caltech} & \textbf{Caltech} \\
& & \textbf{-100} & \textbf{-101} & \textbf{-256} \\
\hline
\hline
\texttt{Chol} & 718 & 691 & 706 & 686 \\
\hline
\texttt{PIChol} & 188 & 167 & 169 & 174 \\
\hline
\texttt{MChol} & 550 & 235 & 545 & 527 \\
\hline
\texttt{SVD} & 9415 & 3489 & 9060 & 9823 \\
\hline
\texttt{t-SVD} & 815 & 858 & 1078 & 1318 \\
\hline
\texttt{r-SVD} & 1140 & 67 & 78 & 111 \\
\hline
\end{tabular}
\caption{Time taken by the six algorithms when $h=16384$ per cross validation fold. All times are reported in seconds.}
\label{table:timing}
\end{table}

\vspace{0.1cm}\noindent Figure~\ref{fig:timing} and Table \ref{table:timing} show the timing results of the Cholesky and SVD-based algorithms. It can be observed that the \texttt{PIChol} has significant speedups over \texttt{Chol} and \texttt{MChol}. Moreover, the \texttt{SVD} and \texttt{t-SVD} are always the slowest. Finally, \texttt{r-SVD} is always the fastest algorithm; however, as we will see in the holdout-error results, \texttt{r-SVD} does not give any useful conclusions for the optimal $\lambda$ value. Figure~\ref{fig:timing} shows the timings for the MNIST data. We obtained similar timing trends for all the four data-sets we used.

\subsection{Hold-out Errors}

\begin{figure}
\centering
\includegraphics[width=0.2375\textwidth]{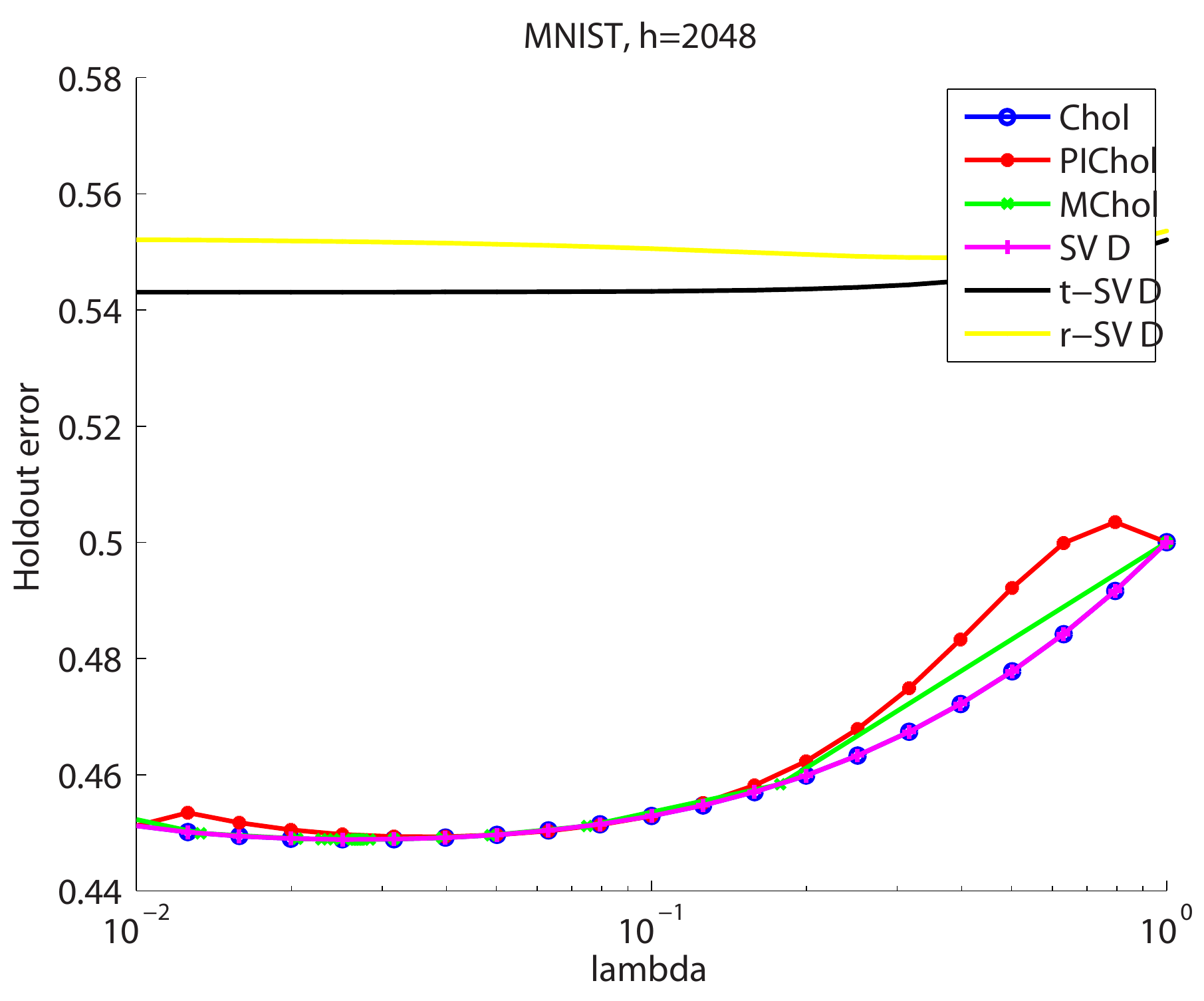}
\includegraphics[width=0.2375\textwidth]{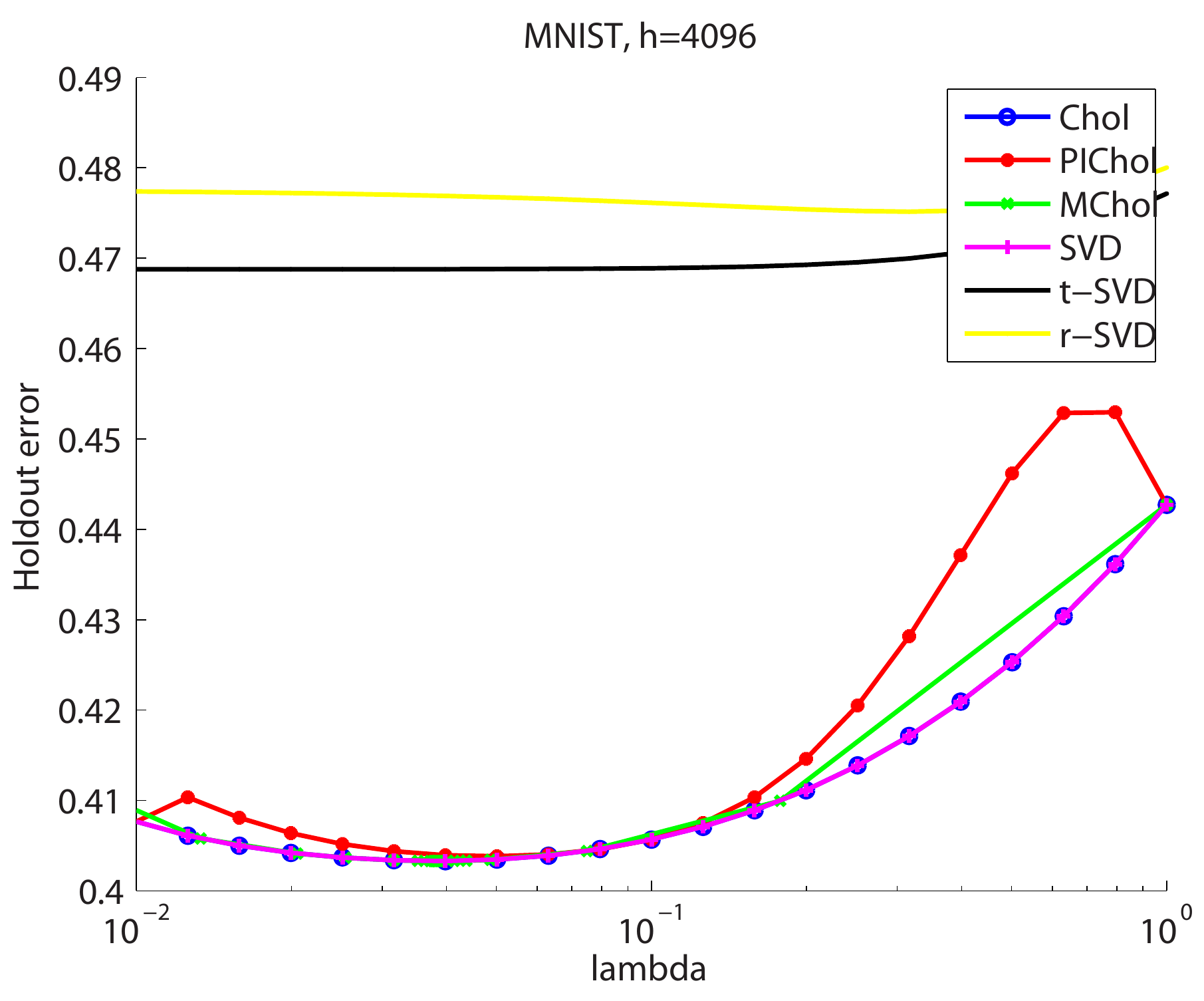}
\includegraphics[width=0.2375\textwidth]{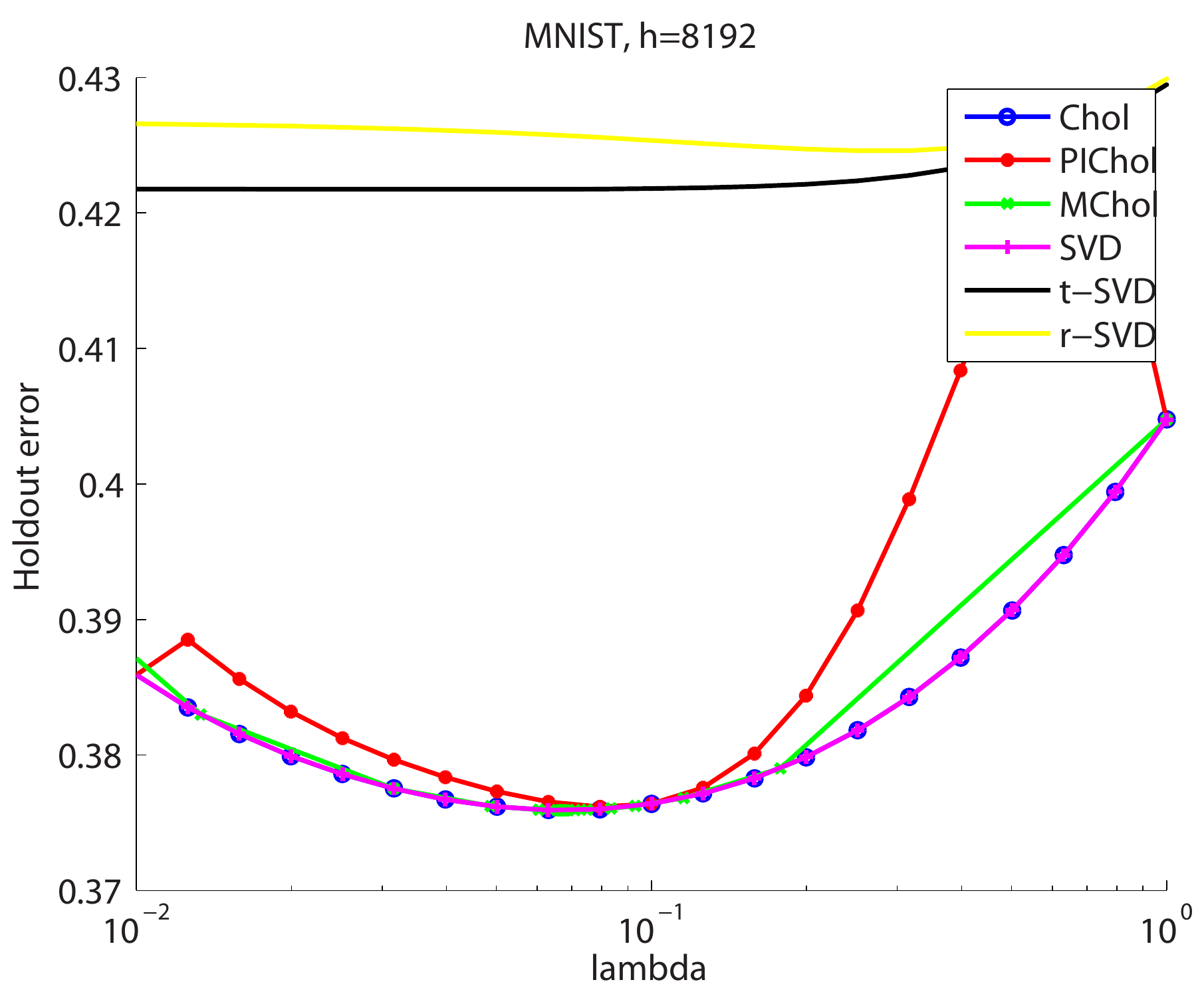}
\includegraphics[width=0.2375\textwidth]{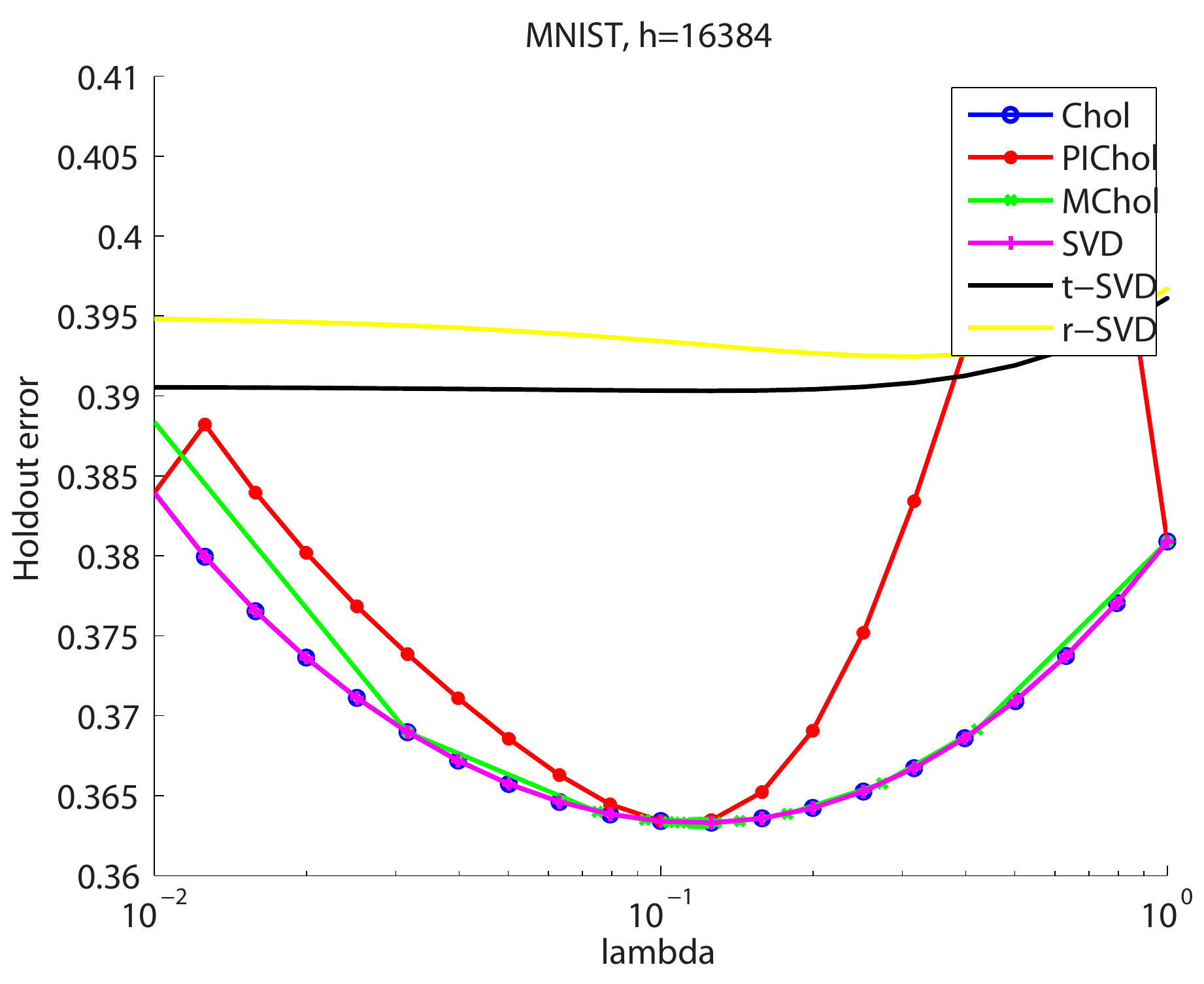}
\caption{Hold-out errors for the six algorithms as a function of $\lambda$ on the MNIST data.}
\label{fig:holdout_zoomin_mnist}
\end{figure}

\begin{figure*}
\centering
\includegraphics[width=0.245\textwidth]{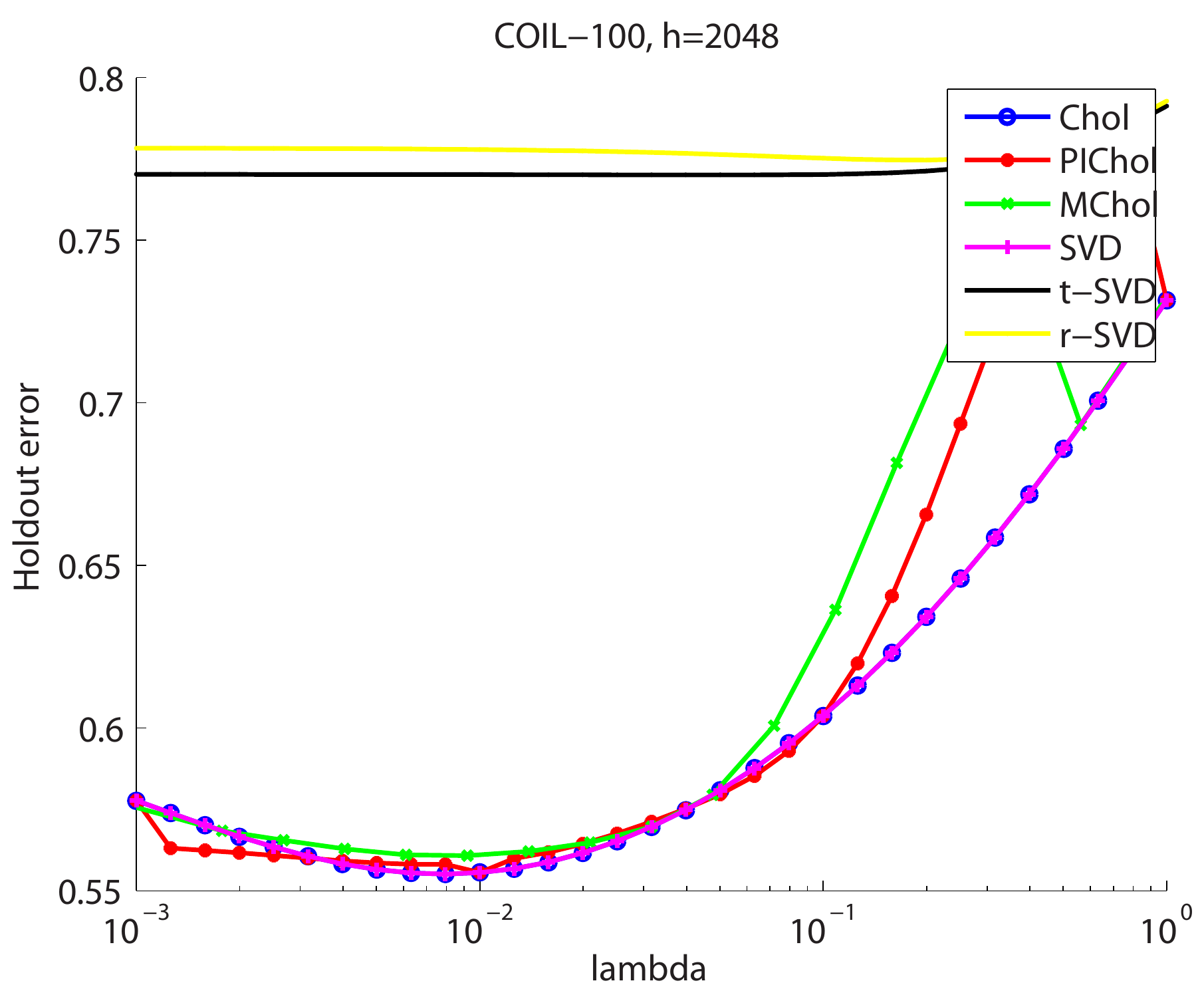}
\includegraphics[width=0.245\textwidth]{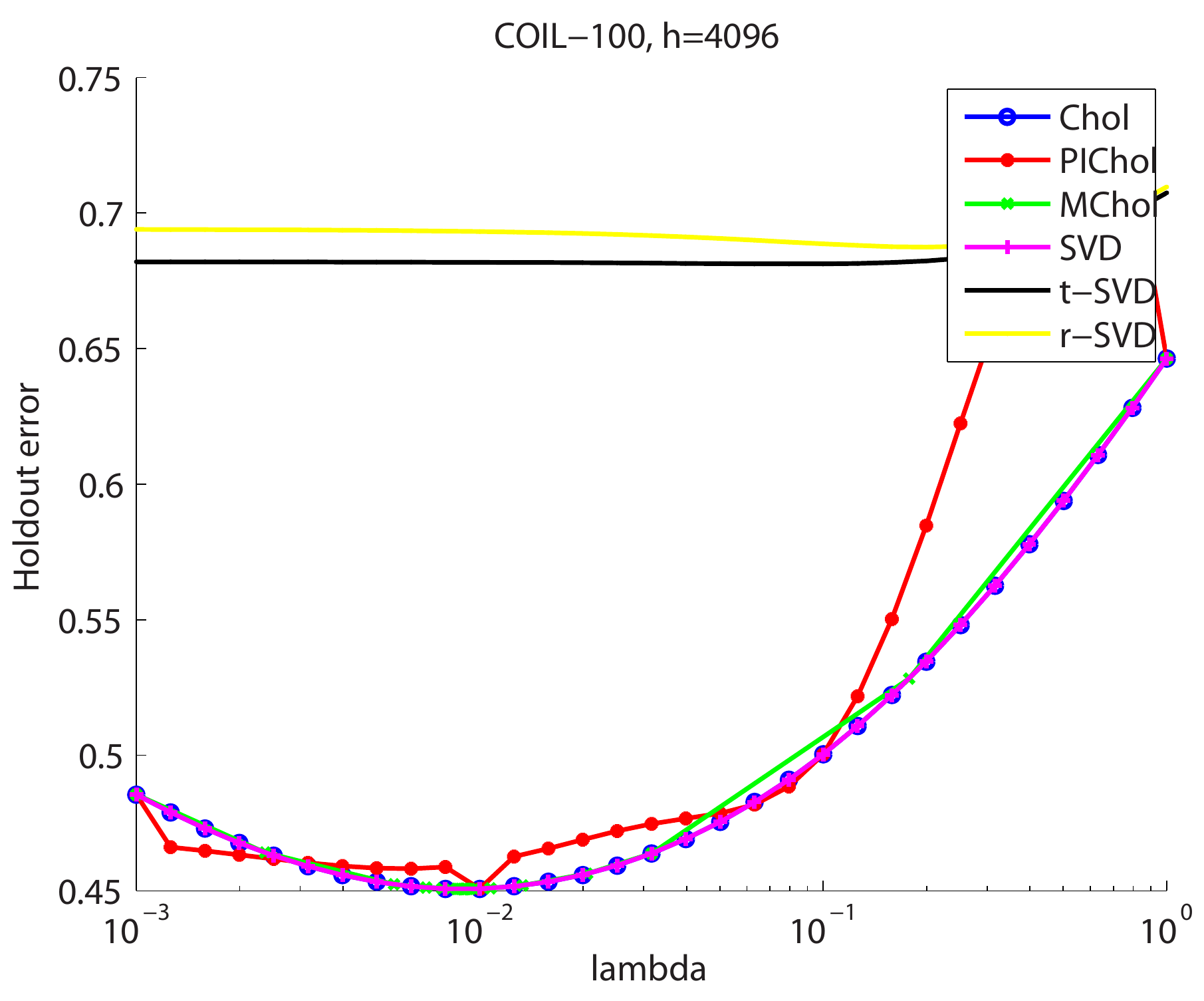}
\includegraphics[width=0.245\textwidth]{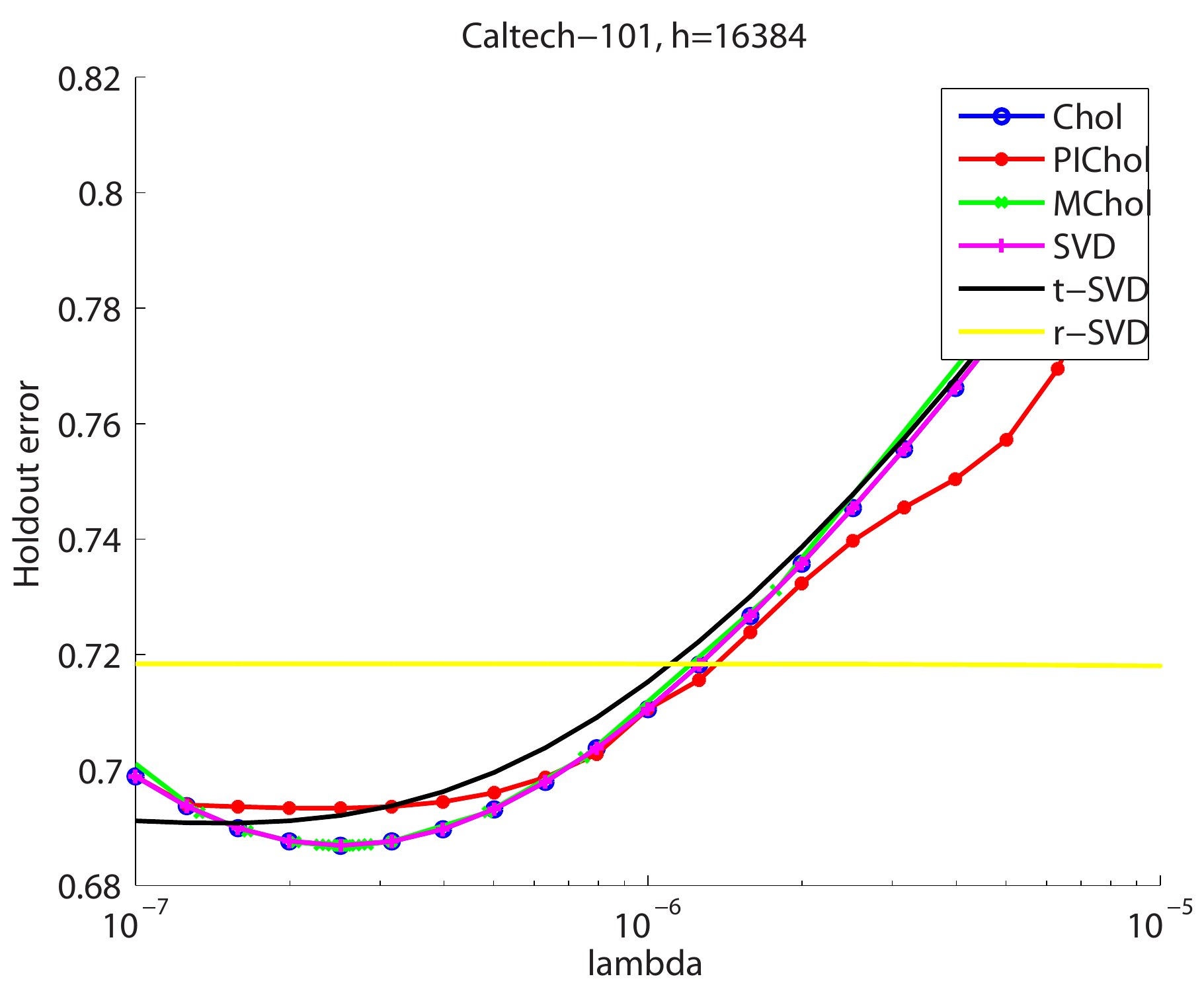}
\includegraphics[width=0.245\textwidth]{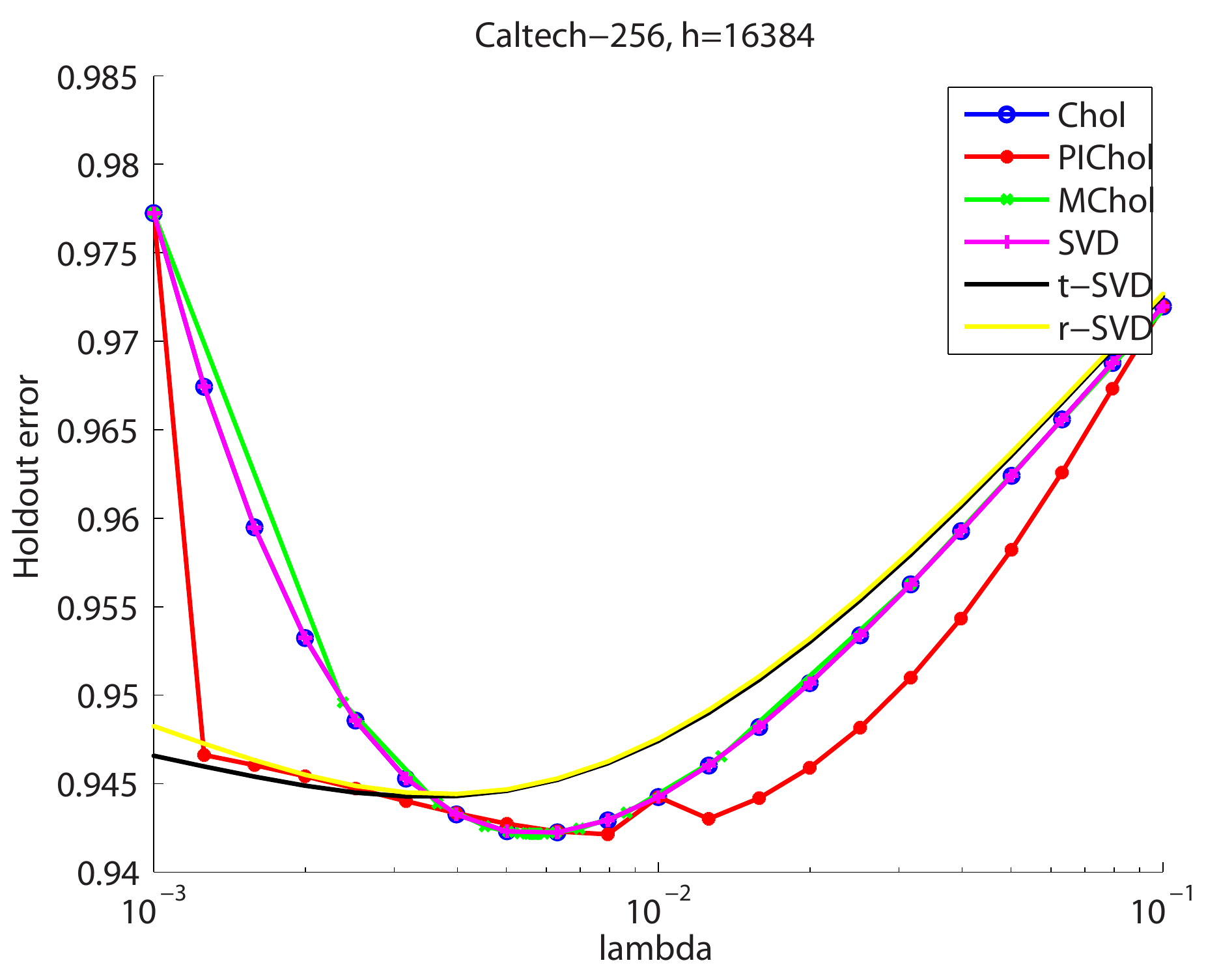}
\caption{\textbf{a}, \textbf{b}-- Hold-out errors for the six algorithms as a function of $\lambda$ on the COIL-100 dataset for projection dimensions of $2048$ and $4096$ respectively. \textbf{c}, \textbf{d}-- Hold-out errors for the six algorithms as a function of $\lambda$ on Caltech 101 and Caltech 256.}
\label{fig:holdout_zoomin_coil_caltech}
\end{figure*}

\begin{table*}
\small
\centering
\begin{tabular}{|c||c|c||c|c||c|c||c|c|}
\hline
\multirow{3}{*}{} & \multicolumn{2}{c||}{\textbf{MNIST}} & \multicolumn{2}{c||}{\textbf{COIL-100}} & \multicolumn{2}{c||}{\textbf{Caltech-101}} & \multicolumn{2}{c|}{\textbf{Caltech-256}} \\
\cline{2-9}
& Minimum ho- & Selected & Minimum ho- & Selected & Minimum ho- & Selected & Minimum ho- & Selected \\
& ldout error & $\lambda$ & ldout error & $\lambda$ & ldout error & $\lambda$ & ldout error & $\lambda$ \\
\hline
\texttt{Chol} & 0.3633 & 0.1259 & 0.4507 & 0.01 & 0.6869 & 2.51e-7 & 0.9422 & 0.0063 \\
\hline
\texttt{PIChol} & 0.3634 & 0.1 & 0.4507 & 0.01 & 0.6934 & 2.51e-7 & 0.9421 & 0.0079 \\
\hline
\texttt{MChol} & 0.3633 & 0.1186 & 0.4506 & 0.009 & 0.6869 & 2.51e-7 & 0.9422 & 0.0057 \\
\hline
\texttt{SVD} & 0.3633 & 0.1259 & 0.4507 & 0.01 & 0.6869 & 2.51e-7 & 0.9422 & 0.0063 \\
\hline
\texttt{t-SVD} & 0.3903 & 0.1259 & 0.6812 & 0.0794 & 0.6908 & 1.58e-7 & 0.9443 & 0.0032 \\
\hline
\texttt{r-SVD} & 0.3925 & 0.3162 & 0.6874 & 0.1995 & 0.7180 & 1.0e-5 & 0.9444 & 0.004 \\
\hline
\end{tabular}
\caption{The minimum hold-out error and the selected $\lambda$ value for the six considered algorithms on four datasets.}
\label{table:holdout}
\end{table*}

\noindent Figure~\ref{fig:holdout_zoomin_mnist} shows the hold-out errors obtained for MNIST data when projected to $2048$, $4096$, $8192$, and $16384$ dimensions respectively. Similarly, Figure~\ref{fig:holdout_zoomin_coil_caltech}-\textbf{a} and \textbf{b} show the hold-out errors for COIL-100 data for $2048$ and $4096$ dimensions, while Figure~\ref{fig:holdout_zoomin_coil_caltech}-\textbf{c} and \textbf{d} show hold-out errors for Caltech 101 and Caltech 256 datasets for $16384$ dimensions each. Table \ref{table:holdout} shows the minimum hold-out error and the selected $\lambda$ value for each algorithm on all the four data sets. Figure~\ref{fig:accuracy_vs_time_coil_caltech} shows the error in the selected $\lambda$'s for \texttt{Chol}, \texttt{PIChol}, and \texttt{MChol} as a function of running time on COIL-100 and Caltech-101 datasets.

\begin{figure}
\centering
\includegraphics[width=0.4\textwidth]{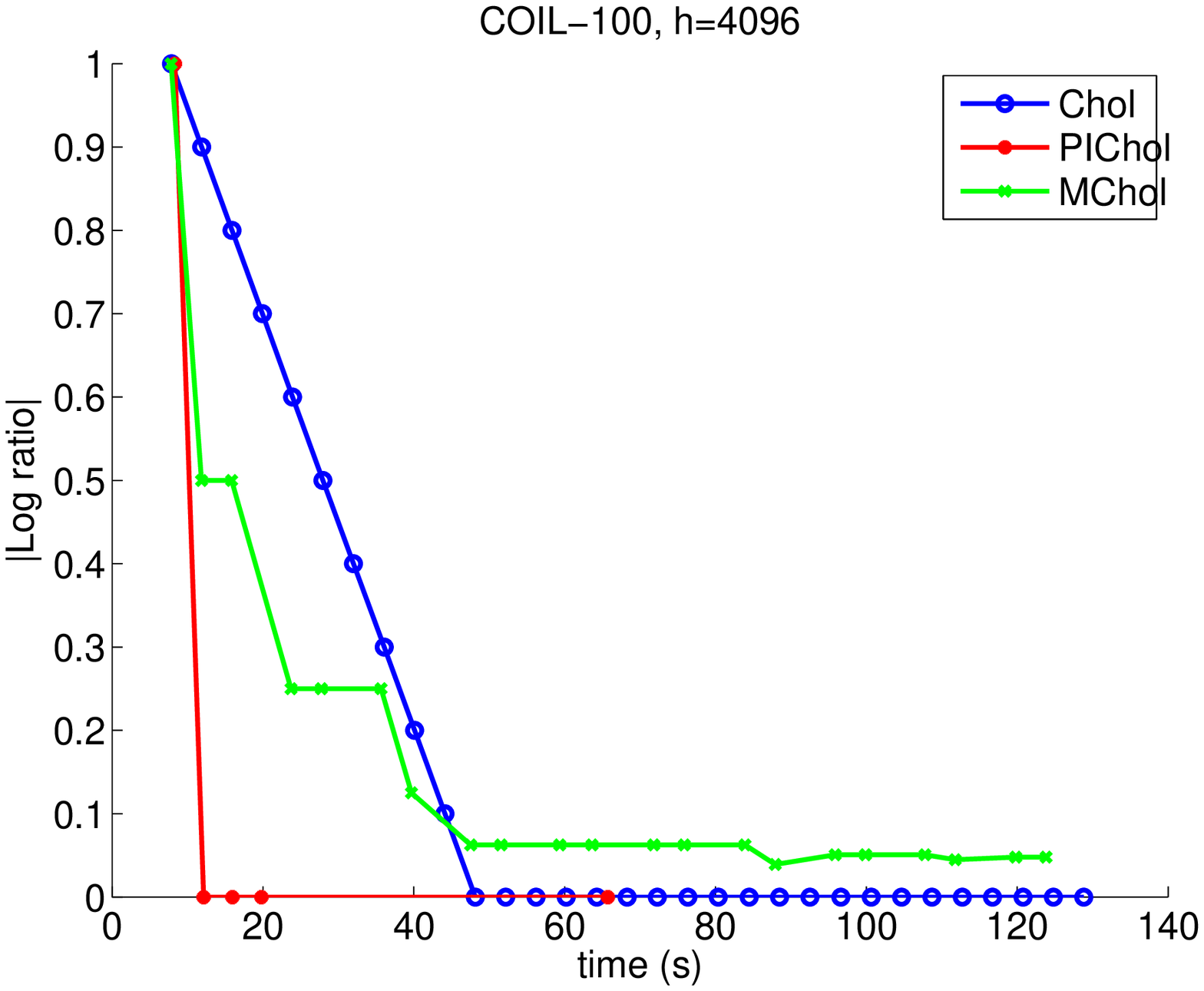}
\includegraphics[width=0.4\textwidth]{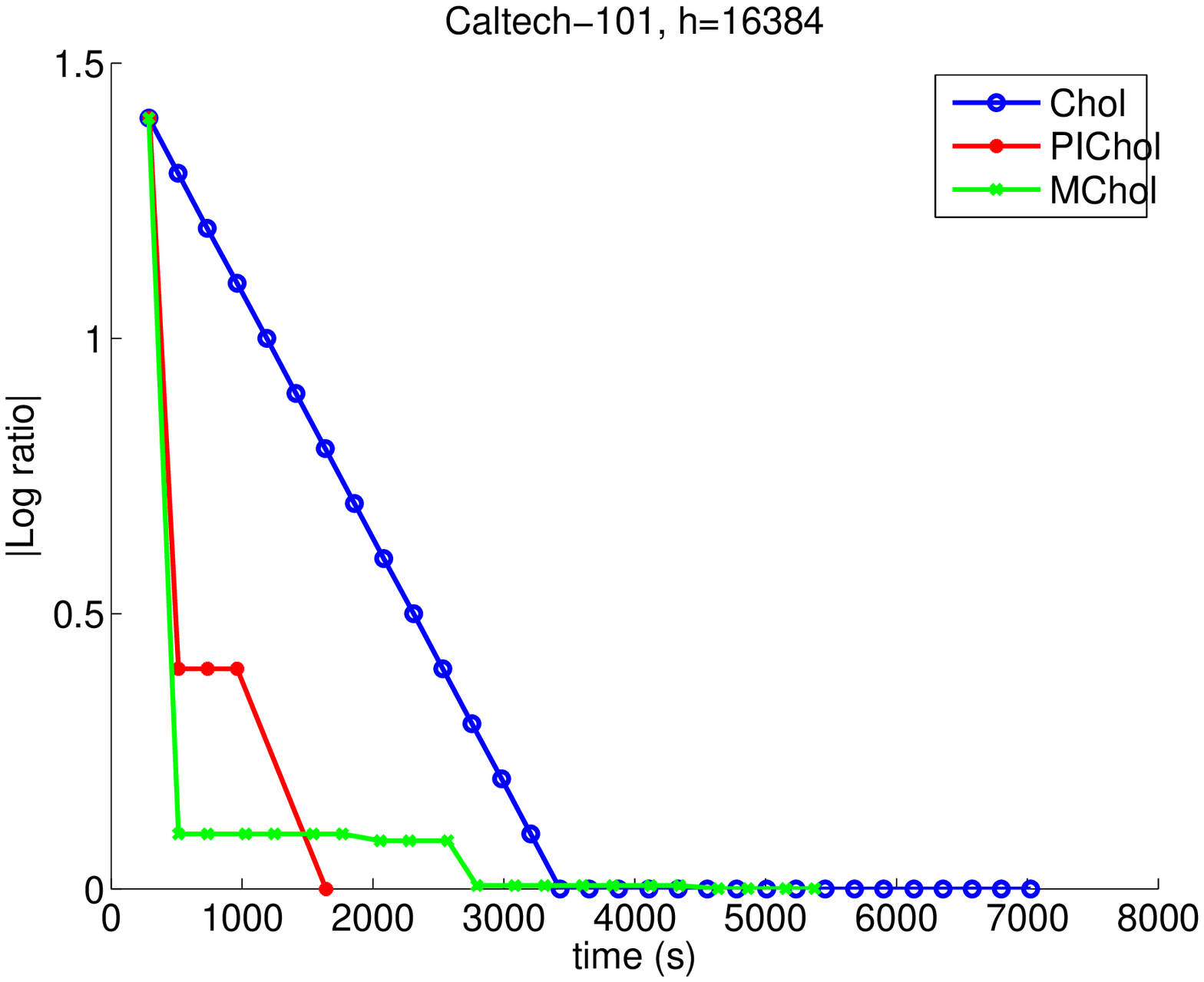}
\caption{Absolute value of the logarithmic ratio between the selected $\lambda$ and the optimal $\lambda$ as \texttt{Chol}, \texttt{PIChol}, and \texttt{MChol} proceed in time.}
\label{fig:accuracy_vs_time_coil_caltech}
\end{figure}

It can be observed that \texttt{PIChol} closely approximates the behavior of \texttt{Chol} and \texttt{SVD}. The approximation for $h=2048,4096$ is better than that for $h=8192,16384$; however, we notice that for the latter two cases with larger $h$, the approximation quality of \texttt{PIChol} is satisfactory when $\lambda$ is close to the optimal $\lambda$ value. This phenomenon justifies the effectiveness of our framework for choosing the optimal regularization parameter. From Table~\ref{table:holdout}, we can see that \texttt{PIChol} and \texttt{MChol} give consistently close approximation of the optimal $\lambda$ value; however, from Figure~\ref{fig:accuracy_vs_time_coil_caltech} and Table~\ref{table:timing}, we conclude that \texttt{MChol} takes much longer time than \texttt{PIChol} to reach the same level of accuracy. Though \texttt{t-SVD} and \texttt{r-SVD} might be faster, they generate very poor approximation of the true hold-out error and the optimal $\lambda$ value, and therefore their efficiency advantage is of little practical use.

Note that another potential way of finding the optimal $\lambda$ from a sparsely sampled set of $\lambda$ values is to interpolate the hold-out error itself based on the hold-out errors computed for the sparsely sampled set. We empirically found that this approach did not yield a good fit to the true hold-out error curve. In the interest of space, we show the detailed results for this scheme (named as \texttt{PINRMSE}) in the supplementary materials.

An alternative way to find the optimal $\lambda$ that corresponds to the minimal
hold-out error is by interpolating a sparse set of hold-out errors directly. We
refer to this approach as \texttt{PINRMSE}. More precisely, \texttt{PINRMSE} is
equivalent to replacing the $g \times \textrm{D}$ matrix $\textrm{\textbf{T}}$
in Algorithm~\ref{alg:chol_interp} with a $g \times 1$ vector $\textrm{\textbf{t}}$, where the entries in $\textrm{\textbf{t}}$ are the hold-out errors that correspond to the sparsely sampled $\lambda$ values. We then interpolate the hold-out errors for the dense set of $\lambda$ values.

Figure~\ref{fig:pinrmse} demonstrates that \texttt{PINRMSE} often gives substantially worse interpolation accuracy than \texttt{PIChol} and could therefore result in dramatically wrong $\lambda$ values. For instance, while \texttt{PINRMSE} achieved the true optimal $\lambda$'s on COIL-100 and Caltech-256 data sets, it selected $\lambda$s that are significantly far away from the optimal values on MNIST and Caltech-101 data sets. On the contrary, \texttt{PIChol} consistently selected the correct $\lambda$'s on all the data sets.

\begin{figure}
\centering
\includegraphics[width=0.425\textwidth]{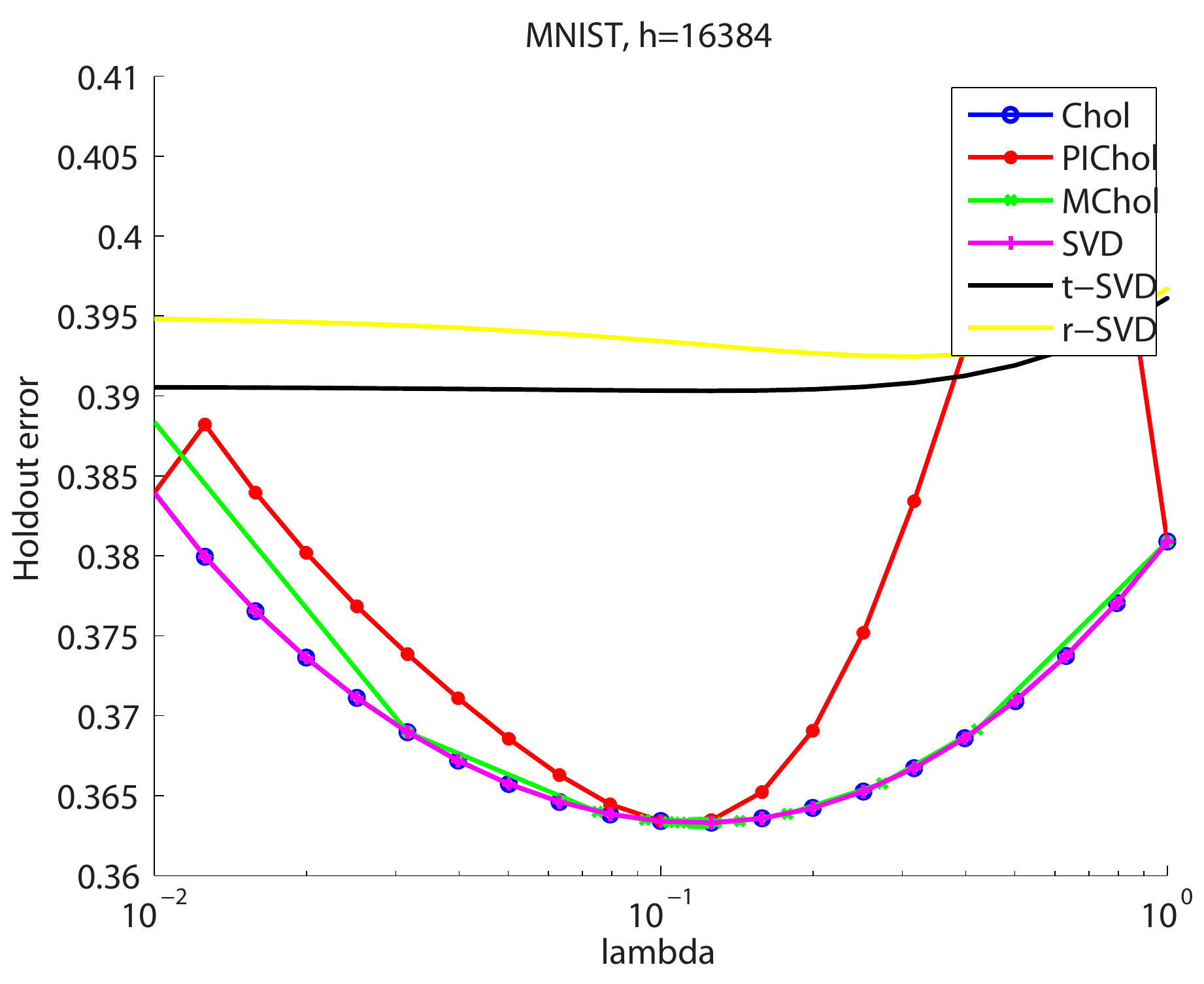}
\includegraphics[width=0.425\textwidth]{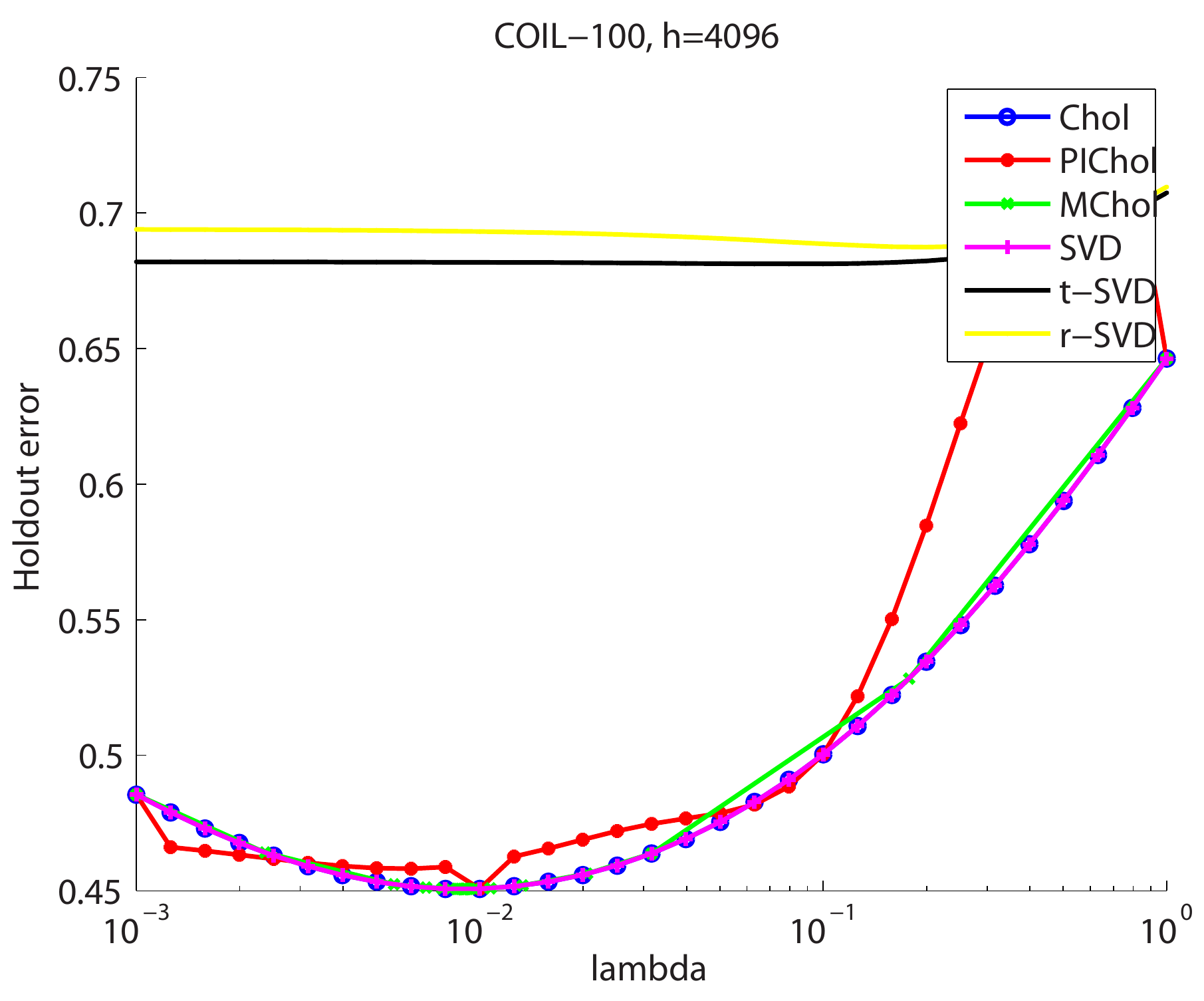}\\
\includegraphics[width=0.425\textwidth]{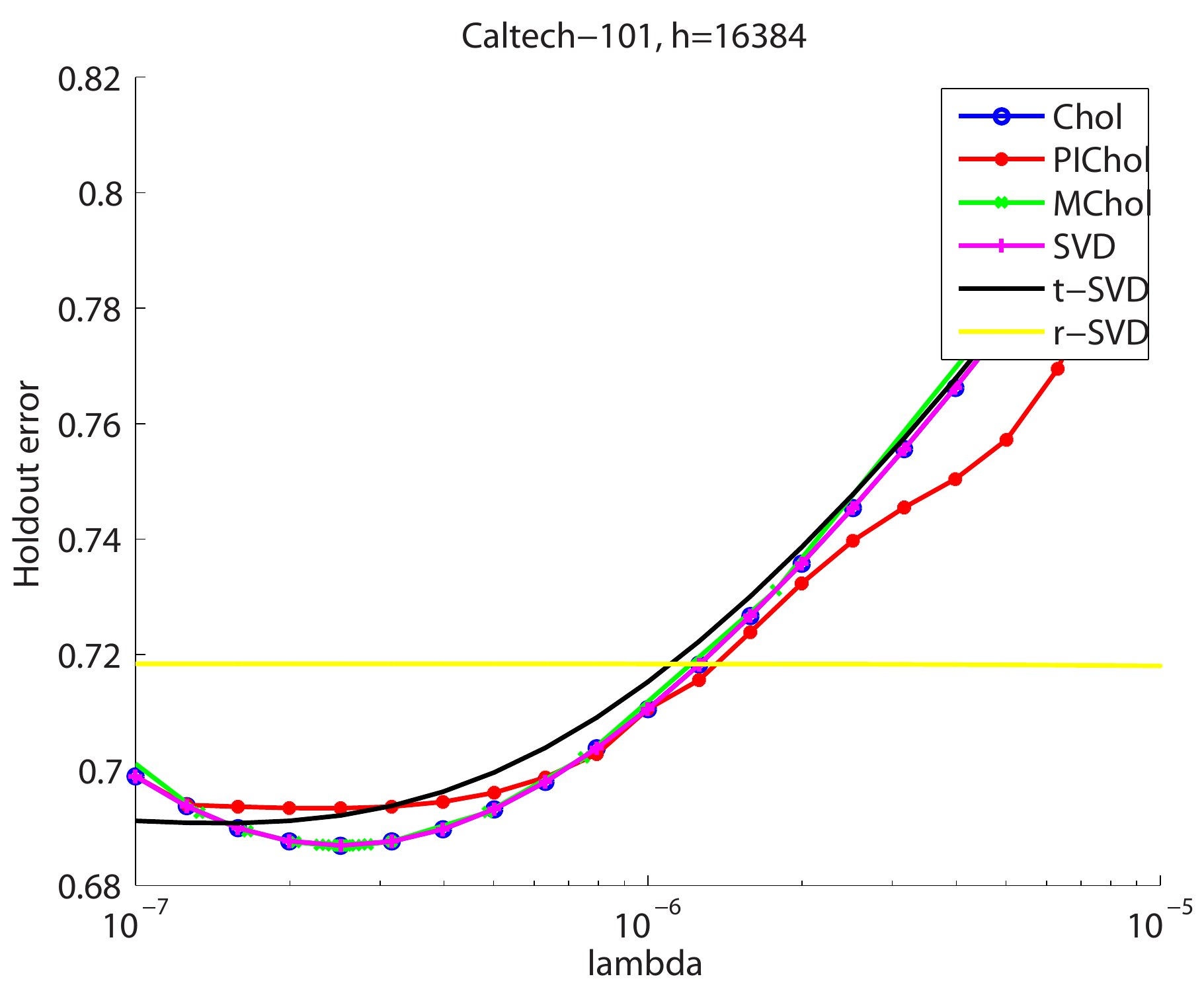}
\includegraphics[width=0.425\textwidth]{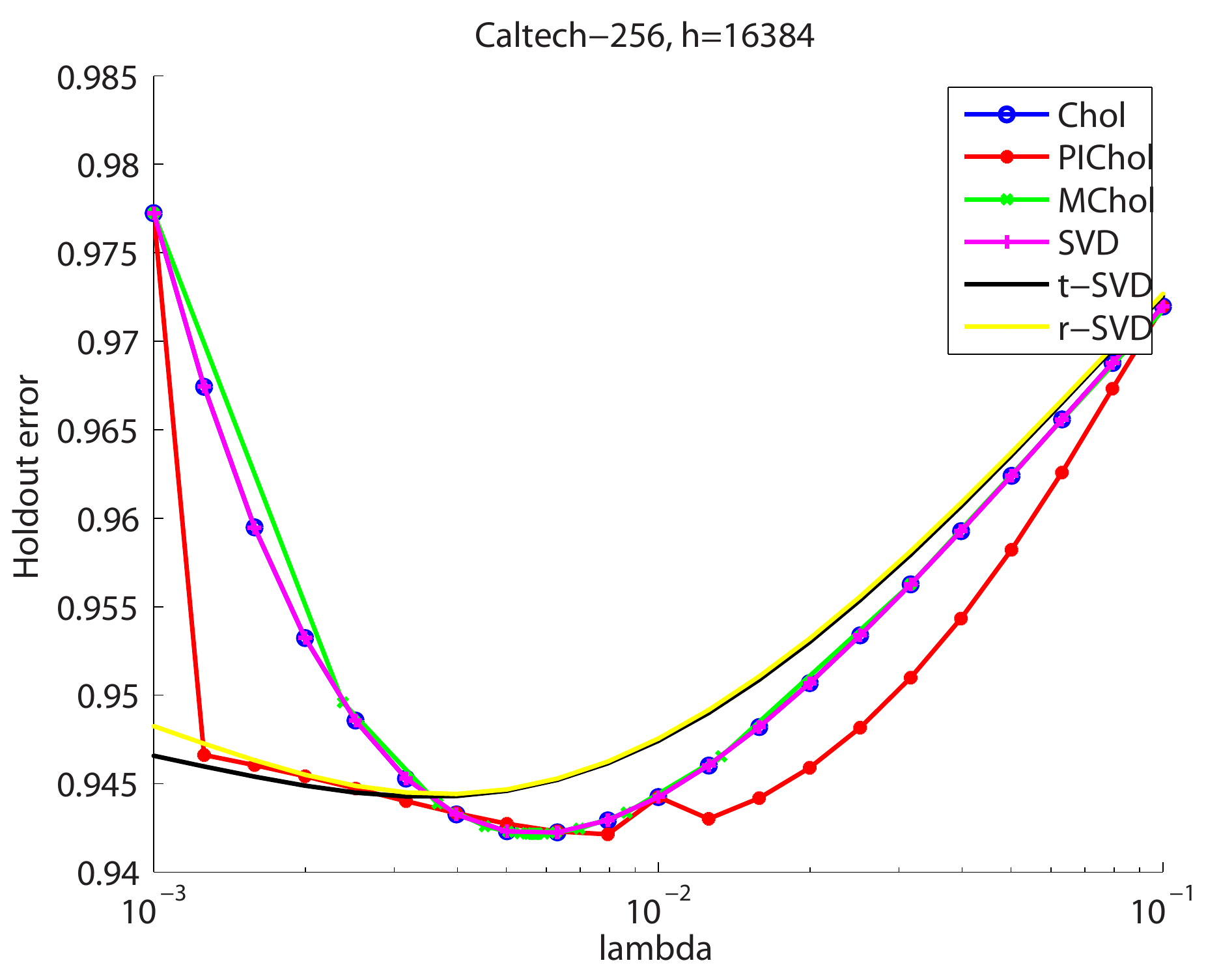}
\caption{Comparisons between \texttt{PINRMSE} and \texttt{Chol}, \texttt{PIChol}, \texttt{SVD}, \texttt{t-SVD}, \texttt{r-SVD}. Both \texttt{PINRMSE} and \texttt{PIChol} use the parameters $g=4$ and $r=2$ in the context of Algorithm 1.}
\label{fig:pinrmse}
\end{figure}

\subsection{Normalized Root Mean Squared Error}
\noindent Figure \ref{fig:nrmse} shows the normalized root mean squared error (NRMSE) for least-squares fit of \texttt{PIChol} on MNIST. Similar trends hold for all considered data-sets. Recall that naively using the mean of target variable implies NRMSE of $1$. Our maximum NRMSE of $0.0457$ hence implies quite high interpolation accuracy.

\begin{figure}
\centering
\includegraphics[width=0.50\textwidth]{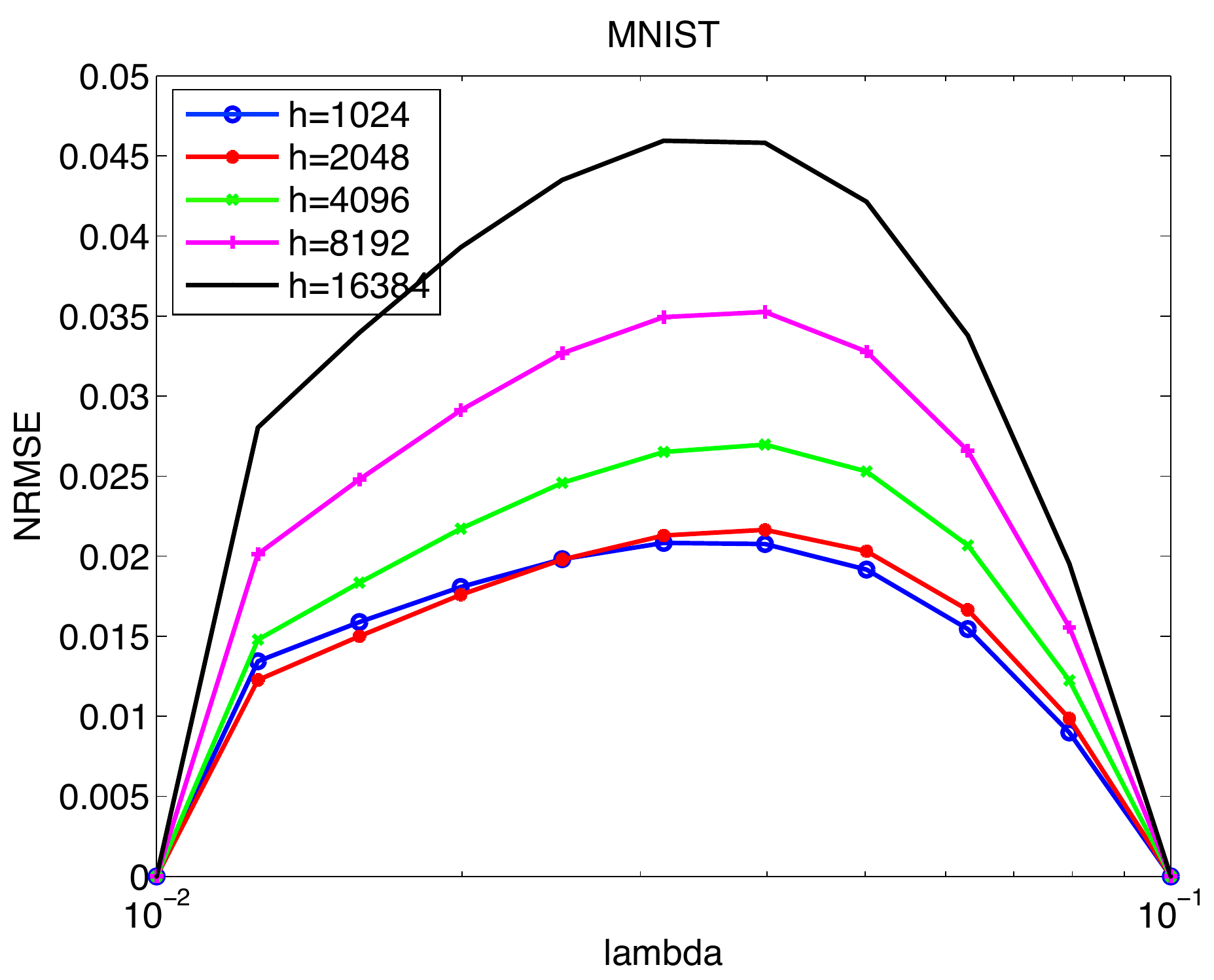}
\caption{NRMSE for $pi$Cholesky on MNIST as a function of regularization parameter $\lambda$.}
\label{fig:nrmse}
\end{figure}

\section{Conclusions \& Future Work}

\vspace{0.1cm}\noindent In this work, we proposed an efficient way to densely interpolate Cholesky factors of Hessian matrix for a sparse set of $\lambda$ values. This idea enabled us to exhaustively explore the space of $\lambda$ values, and therefore optimally minimize the hold-out error while incurring only a fraction of the cost of exact cross-validation. Our key observation was that Cholesky factors for different $\lambda$ values tend to lie on smooth curves that can be approximated accurately using polynomial functions. We theoretically proved this observation and provided an error bound for our approximation. We presented a framework to learn these multiple polynomial functions simultaneously, and proposed solutions for its efficiency challenges. In particular, we proposed a recursive block Cholesky vectorization strategy for efficient vectorization of a triangular matrix. This is a general scheme and is not restricted to the scope of this work.

Currently, we apply the learned polynomial functions within a particular validation fold. Going forward, we intend to use these functions to \textit{warm-start} the learning process in a different fold. This would reduce the number of exact Cholesky factors required in a fold, further improving our performance. We also intend to apply our framework to speed-up regularization in other problems, \textit{e.g.}, matrix completion \cite{mmmf} and sparse coding \cite{sc_admm}.

%

\clearpage

\small{
\bibliographystyle{plain}
\bibliography{sigproc}

\begin{thebibliography}{10}

\bibitem{alpaydin2004introduction}
Ethem Alpaydin.
\newblock {\em Introduction to machine learning}.
\newblock MIT press, 2004.

\bibitem{BrennerScott:FEM}
S.~Brenner and R.~Scott.
\newblock {\em The Mathematical Theory of Finite Element Methods}.
\newblock Springer, 3rd edition, 2007.

\bibitem{bunch1978updating}
James~R Bunch and Christopher~P Nielsen.
\newblock Updating the singular value decomposition.
\newblock {\em Numerische Mathematik}, 31(2):111--129, 1978.

\bibitem{cawley2004fast}
Gavin~C Cawley and Nicola~LC Talbot.
\newblock Fast exact leave-one-out cross-validation of sparse least-squares
  support vector machines.
\newblock {\em Neural networks}, 17(10):1467--1475, 2004.

\bibitem{dietterich1995solving}
Thomas~G. Dietterich and Ghulum Bakiri.
\newblock Solving multiclass learning problems via error-correcting output
  codes.
\newblock {\em arXiv preprint cs/9501101}, 1995.

\bibitem{efron2004least}
Bradley Efron, Trevor Hastie, Iain Johnstone, Robert Tibshirani, et~al.
\newblock Least angle regression.
\newblock {\em The Annals of statistics}, 32(2):407--499, 2004.

\bibitem{fei2007learning}
Li~Fei-Fei, Rob Fergus, and Pietro Perona.
\newblock Learning generative visual models from few training examples: An
  incremental bayesian approach tested on 101 object categories.
\newblock {\em CVIU}, 2007.

\bibitem{golub1965calculating}
Gene Golub and William Kahan.
\newblock Calculating the singular values and pseudo-inverse of a matrix.
\newblock {\em Journal of the Society for Industrial \& Applied Mathematics,
  Series B: Numerical Analysis}, 2(2):205--224, 1965.

\bibitem{golub2012matrix}
Gene~H Golub and Charles~F Van~Loan.
\newblock {\em Matrix computations}.
\newblock Johns Hopkins University Press, 2012.

\bibitem{griffin2007caltech}
Gregory Griffin, Alex Holub, and Pietro Perona.
\newblock Caltech-256 object category dataset.
\newblock 2007.

\bibitem{gu1995downdating}
Ming Gu and Stanley~C Eisenstat.
\newblock Downdating the singular value decomposition.
\newblock {\em SIAM Journal on Matrix Analysis and Applications},
  16(3):793--810, 1995.

\bibitem{hager1989updating}
William~W Hager.
\newblock Updating the inverse of a matrix.
\newblock {\em SIAM review}, 31(2):221--239, 1989.

\bibitem{halko2011finding}
Nathan Halko, Per-Gunnar Martinsson, and Joel~A Tropp.
\newblock Finding structure with randomness: Probabilistic algorithms for
  constructing approximate matrix decompositions.
\newblock {\em SIAM}, 53(2):217--288, 2011.

\bibitem{hansen1987truncatedsvd}
Per~Christian Hansen.
\newblock The truncatedsvd as a method for regularization.
\newblock {\em BIT Numerical Mathematics}, 27(4):534--553, 1987.

\bibitem{horn2012matrix}
Roger~A Horn and Charles~R Johnson.
\newblock {\em Matrix analysis}.
\newblock Cambridge university press, 2012.

\bibitem{kaess2011isam2}
Michael Kaess, Hordur Johannsson, Richard Roberts, Viorela Ila, John Leonard,
  and Frank Dellaert.
\newblock isam2: Incremental smoothing and mapping with fluid relinearization
  and incremental variable reordering.
\newblock In {\em ICRA}, pages 3281--3288. IEEE, 2011.

\bibitem{kar2012random}
Purushottam Kar and Harish Karnick.
\newblock Random feature maps for dot product kernels.
\newblock {\em Journal of Machine Learning Research}, 22:583--591, 2012.

\bibitem{kim2007interior}
Seung-Jean Kim, Kwangmoo Koh, Michael Lustig, Stephen Boyd, and Dimitry
  Gorinevsky.
\newblock An interior-point method for large-scale l1regularized least squares.
\newblock {\em STSP}, 2007.

\bibitem{koutis2012fast}
Ioannis Koutis, Gary~L Miller, and Richard Peng.
\newblock A fast solver for a class of linear systems.
\newblock {\em Communications of the ACM}, 55(10):99--107, 2012.

\bibitem{lazebnik2006beyond}
Svetlana Lazebnik, Cordelia Schmid, and Jean Ponce.
\newblock Beyond bags of features: Spatial pyramid matching for recognizing
  natural scene categories.
\newblock In {\em CVPR}, 2006.

\bibitem{le2013fastfood}
Quoc Le, Tamas Sarlos, and Alexander Smola.
\newblock Fastfood-computing hilbert space expansions in loglinear time.
\newblock In {\em ICML}, pages 244--252, 2013.

\bibitem{lecunmnist}
Yann LeCun and Corinna Cortes.
\newblock The mnist database of handwritten digits, 1998.

\bibitem{svdregression}
John Mandel.
\newblock Use of the singular value decomposition in regression analysis.
\newblock {\em The American Statistician}, 36(1):15--24, 1982.

\bibitem{Marsden}
Jerrold~E. Marsden, Tudor Ratiu, and Ralph Abraham.
\newblock {\em {Manifolds, Tensor Analysis, and Applications}}.
\newblock Springer-Verlag, 3rd edition, 2001.

\bibitem{coil100}
S.~A. Nene, S.~K. Nayar, and H.~Murase.
\newblock Columbia object image library (coil-100).
\newblock Technical Report CUCS-006-96, Columbia University, 1996.

\bibitem{MatrixCookbook}
Kaare~Brandt Petersen and Michael~Syskind Pedersen.
\newblock {\em The Matrix Cookbook}.
\newblock Technical University of Denmark, 2012.

\bibitem{saad2003iterative}
Yousef Saad.
\newblock {\em Iterative methods for sparse linear systems}.
\newblock Siam, 2003.

\bibitem{scholkopf1999advances}
B.~Schlkopf, C.~J.~C. Burges, and A.~J. Smola, editors.
\newblock {\em {Advances in Kernel Methods: Support Vector Learning}}.
\newblock MIT Press, 1999.

\bibitem{mmmf}
Nathan Srebro, Jason Rennie, and Tommi Jaakkola.
\newblock Maximum margin matrix factorization.
\newblock In {\em Advances in Neural Information Processing Systems 17}, pages
  1329--1336, 2004.

\bibitem{sc_admm}
Arthur Szlam, Zhaohui Guo, and Stanley Osher.
\newblock A split bregman method for non-negative sparsity penalized least
  squares with applications to hyperspectral demixing.
\newblock In {\em ICIP}, 2010.

\bibitem{tikhonov1943stability}
Andrey~Nikolayevich Tikhonov.
\newblock On the stability of inverse problems.
\newblock In {\em Dokl. Akad. Nauk SSSR}, volume~39, 1943.

\bibitem{van2003iterative}
Henk~A Van~der Vorst.
\newblock {\em Iterative Krylov methods for large linear systems}, volume~13.
\newblock Cambridge University Press, 2003.

\end{thebibliography}
}

\end{document}